\documentclass{article}
\usepackage[numbers]{natbib}
\usepackage[final]{neurips_2022}
\usepackage[utf8]{inputenc} 
\usepackage[T1]{fontenc}    
\usepackage{hyperref}       
\usepackage{url}            
\usepackage{booktabs}       
\usepackage{amsfonts}       
\usepackage{nicefrac}       
\usepackage{microtype}      
\usepackage[dvipsnames]{xcolor}         

\usepackage{amsmath,amsfonts,bm}

\def\eqref#1{equation~\ref{#1}}

\def\1{\bm{1}}

\DeclareMathAlphabet{\mathsfit}{\encodingdefault}{\sfdefault}{m}{sl}
\SetMathAlphabet{\mathsfit}{bold}{\encodingdefault}{\sfdefault}{bx}{n}

\usepackage{url}
\usepackage{microtype}
\usepackage{graphicx}
\usepackage{booktabs} 
\usepackage{floatrow}
\usepackage{hyperref}
\usepackage{tikz-cd}
\usepackage{amssymb, amsmath, amsthm, amsfonts}
\usepackage{soul}
\usepackage{tikz}
\usepackage{subfigure}
\usepackage{wrapfig}
\usepackage[ruled,vlined]{algorithm2e}
\usepackage{algorithmic}
\usepackage[utf8]{inputenc} 
\usepackage[T1]{fontenc}    
\usepackage{hyperref}       
\usepackage{url}            
\usepackage{nicefrac}       
\usepackage{bbm}
\usepackage{xcolor}
\usepackage{color}
\definecolor{c1}{RGB}{255,75,0}
\definecolor{c2}{RGB}{0,180,255}
\definecolor{LightCyan}{rgb}{0.88,1,1}
\definecolor{Gray}{gray}{0.9}
\usepackage{diagbox}
\usepackage[
    separate-uncertainty = true,
    multi-part-units = repeat
]{siunitx}
\usepackage{microtype}
\usepackage{booktabs} 
\usepackage{multirow}
\usepackage{pifont}
\usepackage{colortbl}
\usepackage{float}
\usepackage{paralist}
\usepackage{array}
\newcolumntype{?}{!{\vrule width 1pt}}
\usetikzlibrary{fit,positioning}
\newcommand\independent{\protect\mathpalette{\protect\independenT}{\perp}}
\def\independenT#1#2{\mathrel{\rlap{$#1#2$}\mkern2mu{#1#2}}}
\newcommand\mydots{\makebox[1em][c]{.\hfil.\hfil.}}
\colorlet{mylinkcolor}{violet}
\colorlet{mycitecolor}{RoyalBlue}
\colorlet{myurlcolor}{RoyalPurple}
\hypersetup{
  citecolor  = mycitecolor,
  linkcolor = myurlcolor,
  urlcolor = myurlcolor,
  colorlinks = true
}
\newcommand{\pa}[1]{\text{pa}(#1)}
\newcommand{\dependent}{ \not\!\perp\!\!\!\perp}

\definecolor{myredcolor}{RGB}{215,48,39}
\definecolor{mygreencolor}{RGB}{26,152,80}
\newcommand{\veryshortarrow}[1][3pt]{\mathrel{%
    \vcenter{\hbox{\rule[-.2pt]{#1}{.4pt}}}%
    \mkern-4mu\hbox{\usefont{U}{lasy}{m}{n}\symbol{41}}}}
\makeatletter
\setbox0\hbox{$\xdef\scriptratio{\strip@pt\dimexpr
\numexpr(\sf@size*65536)/\f@size sp}$}
\newcommand{\scriptveryshortarrow}[1][3pt]{{%
    \vcenter{\hbox{\rule[\scriptratio\dimexpr-.2pt\relax]
    {\scriptratio\dimexpr#1\relax}{\scriptratio\dimexpr.4pt\relax}}}%
    \mkern-4mu\hbox{\let\f@size\sf@size\usefont{U}{lasy}{m}{n}\symbol{41}}}}
\usepackage{pifont}

\newtheorem{Definition}{Definition}
\newtheorem{proposition}{Proposition}

\title{Factored Adaptation for Non-stationary Reinforcement Learning}

\author{
    Fan Feng$^1$, Biwei Huang$^2$, Kun Zhang$^{2,3}$, Sara Magliacane$^{4,5}$\\
    $^1$City University of Hong Kong
    $^2$Carnegie Mellon University\\
    $^3$Mohamed bin Zayed University of Artificial Intelligence\\
    $^4$University of Amsterdam
    $^5$MIT-IBM Watson AI Lab\\
    {\small \texttt{\{ffeng1017,sara.magliacane\}@gmail.com}, \texttt{biweih@andrew.cmu.edu}}, {\small \texttt{kunz1@cmu.edu}}}

\begin{document}
    \maketitle
    \begin{abstract}
        Dealing with non-stationarity in environments (e.g., in the transition dynamics) and objectives (e.g., in the reward functions) is a challenging problem that is crucial in real-world applications of reinforcement learning (RL).
        While most current approaches model the changes as a single shared embedding vector, we leverage insights from the recent causality literature to model non-stationarity in terms of individual latent change factors, and causal graphs across different environments.
        In particular, we propose Factored Adaptation for Non-Stationary RL (FANS-RL), a factored adaption approach
        that learns jointly both the causal structure in terms of a factored MDP, and a factored representation of the individual time-varying change factors. We prove that under standard assumptions, we can completely recover the causal graph representing the factored transition and reward function, as well as a partial structure between the individual change factors and the state components.
        Through our general framework, we can consider general non-stationary scenarios with different  function types and changing frequency, including changes across episodes and within episodes.
        Experimental results demonstrate that FANS-RL outperforms existing approaches in terms of return, compactness of the latent state representation, and robustness to varying degrees of non-stationarity.
    \end{abstract}

    \section{Introduction}
    Learning a stable policy under non-stationary environments is a long-standing challenge in Reinforcement learning (RL)~\citep{sutton2018reinforcement,dulac2021challenges, padakandla2021survey}.
    While most RL approaches assume stationarity, in many real-world applications of RL there can be changes in the dynamics or the reward function, both \emph{across} different episodes and \emph{within} each episode.
    Recently, several works adapted Meta-RL methods to learn sequences of non-stationary tasks \citep{alshedivat2018continuous,ijcai2021-399}. However, the continuous MAML \citep{finn2017model} adaptation for non-stationary RL \citep{alshedivat2018continuous} does not explicitly model temporal changing components, while TRIO~\citep{ijcai2021-399} needs to meta-train the model on a set of non-stationary tasks. LILAC~\citep{xie2020deep} and ZeUS \cite{sodhani2021block} leverage latent variable models to directly model the change factors in the environment in a shared embedding space.
    In particular, they consider families of MDPs indexed by a single latent parameter. In this paper, we argue that disentangling the changes as separate latent parameters and modeling the process with a factored representation improves the efficiency of adapting to non-stationarity.

    In particular, we leverage insights from the causality literature \citep{CDNOD_Huang20, DA_Zhang20} that model non-stationarity in terms of individual latent change factors and causal graphs across different environments.
    We propose Factored Adaptation for Non-Stationary RL (FANS-RL), a factored adaptation framework that jointly learns the causal structure of the MDP and a factored representation of the individual change factors, allowing for changes at discrete timepoints and continuously varying environments.
    While we provide a specific architecture (FN-VAE), the theoretical framework of FANS-RL can be implemented with different architectures and combined with various RL algorithms.
    We formalize our setting as a \emph{Factored Non-stationary MDP} (FN-MDP),  which combines a Factored-MDP ~\cite{boutilier2000stochastic, kearns1999efficient, osband2014near} with latent change factors that evolve in time following a Markov process.

    We build upon the AdaRL framework \cite{huang2021adarl}, a recently proposed fast adaptation approach. AdaRL learns a factored representation that explicitly models changes (i.e., domain-specific components) in observation, dynamics and reward functions across a set of source domains. An optimal policy learnt on the source domains can then be adapted to a new target domain simply by identifying a low-dimensional change factor, without any additional finetuning.
    FANS-RL extends AdaRL from the stationary case with constant change factors to a general non-stationary framework.
    Specifically, FANS-RL learns the low-dimensional and time-evolving representations $\boldsymbol{\theta}_t^s$ and $\boldsymbol{\theta}^r_t$ that fully capture the non-stationarity of dynamics and rewards, allowing for continuous and discrete changing functions, both \emph{within-episode} and \emph{across-episode}.
    Our main contributions can be summarized as:
    \begin{compactitem}
        \item We formalize FN-MDPs, a unified factored framework that can handle many non-stationary settings, including discrete and continuous changes, both within and across episodes. We prove that, under standard assumptions, the causal graph of the transition and reward function is identifiable, while we can recover a partial structure for the change factors.
        \item We introduce Factored Adaptation for Non-Stationary RL (FANS-RL), a general non-stationary RL approach that interleaves model estimation of an FN-MDP and policy optimization. We also describe FN-VAE, an example architecture for learning FN-MDPs.
        \item We evaluate FANS-RL on simulated benchmarks for continuous control and robotic manipulation tasks and show it outperforms the state of the art on the return, compactness of the latent space representation and robustness to varying degrees of non-stationarity.
    \end{compactitem}

    \section{Factored Non-stationary MDPs}
    \label{Sec: FMDP}
    To model different types of non-stationarity in a unified and factored way, we propose Factored Non-stationary Markov Decision Processes (FN-MDPs). FN-MDPs are an augmented form of a factored MDPs~\cite{boutilier2000stochastic, kearns1999efficient, osband2014near} with latent change factors that evolve over time following a Markov process. Since the change factors are latent, FN-MDPs are partially observed. We define them as:

    \begin{Definition}
        A Factored Non-stationary Markov Decision Process (FN-MDP) is a tuple $\left(\mathcal{S}, \mathcal{A}, \Theta^s,\Theta^r, \gamma, \mathcal{G}, \mathbb{P}_s, \mathcal{R}, ,  \mathbb{P}_{\theta^r}, \mathbb{P}_{\theta^s} \right)$, where $\mathcal{S}$ is the state space, $\mathcal{A}$ the action space, $\Theta^s$ the space of the change factors for the dynamics,  $\Theta^r$ the space of the reward change factors and $\gamma$ the discount factor.
        We assume $\mathcal{G}$ is a Dynamic Bayesian Network over $\{ s_{1,t}, \mydots, s_{d,t}, a_{1,t}, \mydots, a_{m,t}, r_t, \theta^{\boldsymbol{s}}_{1,t}, \mydots, \theta^{\boldsymbol{s}}_{p,t}, \\ \theta^{r}_{1,t}, \mydots, \theta^{r}_{q,t}\}$, where $d$, $m$, $p$, and $q$ are the dimensions of states, action, change factors on dynamics and reward, respectively.
        We define the factored state transition distribution $\mathbb{P}_s$ as:
        \begin{align*}
            \mathbb{P}_s (\boldsymbol{s}_t | \boldsymbol{s}_{t-1}, \boldsymbol{a}_{t-1}, \boldsymbol{\theta}^s_t) = \prod_{i=1}^d \mathbb{P}_s ( s_{i,t} | \pa{s_{i,t}})
        \end{align*}
        where $\pa{s_{i,t}}$ denotes the causal parents of $s_{i,t}$ in $\mathcal{G}$, which are a subset of the dimensions of $\boldsymbol{s}_{t-1}$, $ \boldsymbol{a}_{t-1}$ and $\boldsymbol{\theta_t^s}$.
        Note that the action $\boldsymbol{a}_{t-1}$ is a vector of $m$ dimensions in our setting.
        We assume a given initial state distribution $\mathbb{P}_s (\boldsymbol{s}_0)$.
        Similarly, we define the reward function $\mathcal{R}$ as a function of the parents of $r_t$ in $\mathcal{G}$, i.e., $\mathcal{R}( \mathbf{s}_t, \mathbf{a}_t, \boldsymbol{\theta^r}_t) = \mathcal{R}(\pa{r_t})$, where $\pa{r_t}$ are a subset of dimensions of $\mathbf{s}_t, \mathbf{a}_t,$ and $\boldsymbol{\theta^r}_t$.
        We define the factored latent change factors transition distributions $\mathbb{P}_{\theta^s}$ and $\mathbb{P}_{\theta^r}$ as:
        \begin{align*}
            \mathbb{P}_{\theta^s} (\boldsymbol{\theta^s}_t |  \boldsymbol{\theta^s}_{t-1}) &= \prod_{j=1}^p \mathbb{P}_{\theta^s}  ( \theta^s_{j,t} | \pa{\theta^s_{j,t}}), &
            \mathbb{P}_{\theta^r} (\boldsymbol{\theta^r}_t | \boldsymbol{\theta^r}_{t-1}) &= \prod_{k=1}^q \mathbb{P}_{\theta^r}  ( \theta^r_{k,t} | \pa{\theta^r_{k,t}})
        \end{align*}
        where $\pa{\theta^s_{j,t}}$ are a subset of the dimensions of $\boldsymbol{\theta^s}_{t-1}$, while $\pa{\theta^r_{k,t}}$ are a subset of dimensions of $\boldsymbol{\theta^r}_{t-1}$. We assume the initial distributions $\mathbb{P}_{\theta^s} (\boldsymbol{\theta^s}_0)$ and $\mathbb{P}_{\theta^r} (\boldsymbol{\theta^r}_0)$ are given.
    \end{Definition}

    We show an example DBN representing the graph $\mathcal{G}$ of an FN-MDP in Fig.~\ref{fig:vae}(a).
    Since we are interested in learning the graphical structure of the FN-MDP, as well as identifying the values of the latent change factors from data, we
    describe a generative process of an FN-MDP environment. In particular, we assume that the graph $\mathcal{G}$ is time-invariant throughout the non-stationarity, and there are no unobserved confounders and instantaneous causal effects in the system.
    We will learn a set of binary masks $\boldsymbol{c}^{\cdot \scriptveryshortarrow \cdot}$ and $\boldsymbol{C}^{\cdot \scriptveryshortarrow \cdot}$ that are the indicators for edges in $\mathcal{G}$.

    \textbf{Generative environment model.}
    We adapt the generative model in AdaRL \citep{huang2021adarl} across $k$ different domains to a time-varying setting on a single domain. We assume the generative process of the environment
    at timestep $t$ in terms of the transition function for each dimension $i=1,\mydots, d$ of $\mathbf{s}_t$ is:
    
    \begin{equation}
        \label{eq: adarl1}
        s_{i, t}=f_{i}(\boldsymbol{c}_{i}^{\boldsymbol{s} \scriptveryshortarrow \boldsymbol{s}} \odot \boldsymbol{s}_{t-1},
        \boldsymbol{c}_{i}^{\boldsymbol{a} \scriptveryshortarrow \boldsymbol{s}} \odot \boldsymbol{a}_{t-1},
        \boldsymbol{c}_{i}^{\boldsymbol{\theta}_{t} \scriptveryshortarrow \boldsymbol{s}} \odot \boldsymbol{\theta}_{t}^{\boldsymbol{s}}, \epsilon_{i, t}^{s})
    \end{equation}
    
    where $\odot$ is the element-wise product, $f_i$ are non-linear functions and $\epsilon^s_{i,t}$ is an i.i.d. random noise. The binary mask $\boldsymbol{c}_{i}^{\boldsymbol{s} \scriptveryshortarrow \boldsymbol{s}} \in \{0, 1\}^d$ represents which state components $s_{j,t-1}$ are used in the transition function of $s_{i,t}$.
    Similarly, $\boldsymbol{c}_{i}^{\boldsymbol{a} \scriptveryshortarrow \boldsymbol{s}} \in \{0, 1\}^m$ indicates whether the action directly affects $s_{i,t}$.
    The change factor $\boldsymbol{\theta}_{t}^{\boldsymbol{s}} \in \mathbb{R}^p$
    encodes any change in the dynamics. The binary mask $\boldsymbol{c_{i}}^{\boldsymbol{\theta}_{t} \scriptveryshortarrow \boldsymbol{s}} \in \{0,1\}^p$ represents which of the components of $\boldsymbol{\theta}_{t}^{\boldsymbol{s}}$ influence $s_{i,t}$. We model the reward function as:
    
    \begin{equation}
        r_{t}=h (\boldsymbol{c}^{\boldsymbol{s} \scriptveryshortarrow r} \odot \boldsymbol{s}_{t}, \boldsymbol{c}^{\boldsymbol{a} \scriptveryshortarrow r} \odot \boldsymbol{a}_{t}, \boldsymbol{\theta}_{t}^{r}, \epsilon_{t}^{r} )
    \end{equation}
    
    where $\boldsymbol{c}^{\boldsymbol{s} \scriptveryshortarrow \boldsymbol{r}} \in \{0, 1\}^d$, $\boldsymbol{c}^{\boldsymbol{a} \scriptveryshortarrow \boldsymbol{s}} \in \{0, 1\}^m$, and $\epsilon_{t}^{r}$ is an i.i.d. random noise. The change factor $\boldsymbol{\theta}_{t}^{r} \in \mathbb{R}^q$
    encodes any change in the reward function.
    The binary masks $\boldsymbol{c}^{\cdot \veryshortarrow \cdot}$ can be seen as indicators of edges in the DBN $\mathcal{G}$.
    In AdaRL, all change factors are assumed to be constant in each domain.
    Since in this paper we allow the change parameters to evolve in time, we introduce two additional equations:
    \begin{figure}[t]
        \centering
        \includegraphics[width=0.9\linewidth]{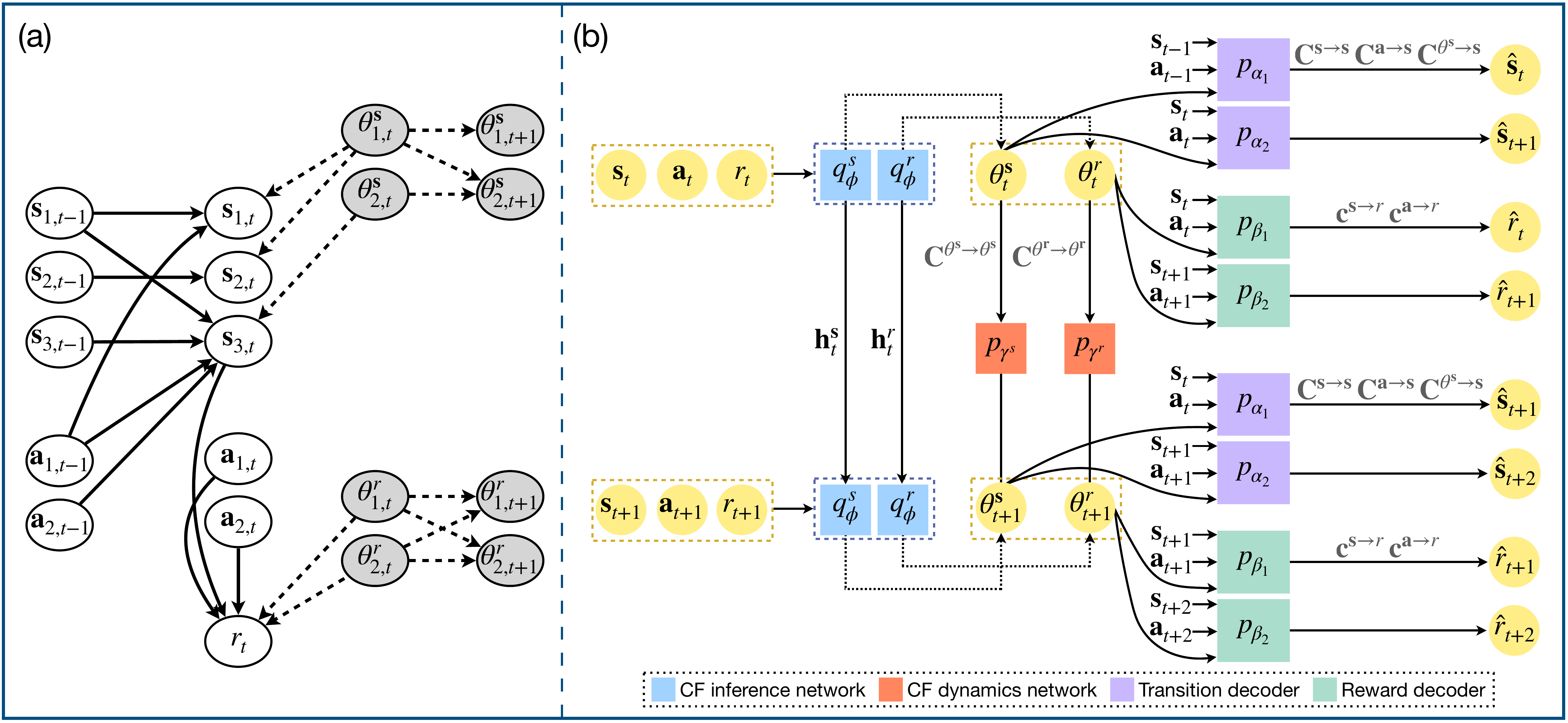}
        \caption{(a). A graphical representation of an FN-MDP. For readability, we only illustrate a subsection of dimensions of states, actions, and latent change factors. The shaded variables are unobserved; (b). The architecture of FN-VAE, which learns the generative model, explained in Sec.~\ref{sec:FAEVAE}.}
        \label{fig:vae}
    \end{figure}
    
    \begin{equation}
        \begin{array}{ll}
            \theta^s_{j,t} & = g^s (\boldsymbol{c}^{\boldsymbol{\theta^s} \scriptveryshortarrow \boldsymbol{\theta^s}}_{j} \odot \boldsymbol{\theta^s}_{t-1} , \epsilon^{\theta^s}_t )                         \\
            \theta^r_{k,t} & = g^r (\boldsymbol{c}^{\boldsymbol{\theta}^r \scriptveryshortarrow \boldsymbol{\theta}^r}_{k} \odot \boldsymbol{\theta}^r_{t-1} , \epsilon^{\theta^r}_t ) \label{Eq: conti_model}
        \end{array}
    \end{equation}
    
    for $i = 1,\mydots, d$,  $j = 1,\mydots, p$, $k=1, \mydots, q$, and $g^s$, and $g^r$ are non-linear functions.
    We assume the binary masks $\boldsymbol{c}^{\cdot \scriptveryshortarrow \cdot}$ are stationary across timesteps and so are the $\epsilon_{i, t}^{s}$, $\epsilon_{t}^{r}$, $\epsilon^{\theta^s}_t$ and $ \epsilon^{\theta^r}_t$, the i.i.d. random noises. Although $\boldsymbol{c}^{\cdot \scriptveryshortarrow \cdot}$ and $\epsilon$ are stationary, we model the changes in the functions and some changes in the graph structure through $\boldsymbol{\theta}$. For example a certain value of $\boldsymbol{\theta}_t^r$ can switch off the contribution of some of the state or action dimensions in the reward function, or in other words nullify the effect of some edges in $\mathcal{G}$. Similarly the contribution of the noise distribution to each function can be modulated via the change factors. On the other hand, this setup does not allow adding edges that are not captured by the binary masks $\boldsymbol{c}^{\cdot \scriptveryshortarrow \cdot}$.
    We group the binary masks in the matrices $\boldsymbol{C}^{\boldsymbol{s} \scriptveryshortarrow \boldsymbol{s}} := [\boldsymbol{c}^{\boldsymbol{s} \scriptveryshortarrow \boldsymbol{s}}_i]_{i=1}^d$,
    $\boldsymbol{C}^{\theta^{s} \scriptveryshortarrow \boldsymbol{s}} := [\boldsymbol{c}^{\theta^{s} \scriptveryshortarrow \boldsymbol{s}}_i]_{i=1}^d$,
    and
    $\boldsymbol{C}^{\boldsymbol{a} \scriptveryshortarrow \boldsymbol{s}} := [\boldsymbol{c}^{\boldsymbol{a} \scriptveryshortarrow \boldsymbol{s}}_i]_{i=1}^d$.
    Similarly, we also group the binary vectors in the dynamics of the latent change factors in matrices $\boldsymbol{C}^{\boldsymbol{\theta^s} \scriptveryshortarrow \boldsymbol{\theta^s}} := [\boldsymbol{c}^{\boldsymbol{\theta^s} \scriptveryshortarrow \boldsymbol{\theta^s}}_j]_{j=1}^p$ and $\boldsymbol{C}^{\boldsymbol{\theta}^r \scriptveryshortarrow \boldsymbol{\theta}^r} := [\boldsymbol{c}^{\boldsymbol{\theta}^r \scriptveryshortarrow \boldsymbol{\theta}^r}_k]_{k=1}^q$.
    Since latent change factors $\boldsymbol{\theta^s}$ and $\boldsymbol{\theta}^r$ follow a Markov process based on $g^s$ and $g^r$, we can consider different types of changes by varying the form of $g^s$ and $g^r$, generalizing the approaches in literature. We can also model concurrent changes in dynamics and reward, including different types of changes, e.g. a continuous $g^s$ and a piecewise-constant $g^r$.

    \textbf{Compact representations. } 
    \citet{huang2021adarl} show that the only dimensions of the state and change factors useful for policy learning are those that eventually affect the reward. These dimensions are called \emph{compact representations} and are defined as the dimensions of the state and change factors with a path (i.e. a sequence of edges $\to$) to the present or future reward $r_{t+\tau}$ for $ \tau \geq 0$ in the DBN $\mathcal{G}$:
    \begin{align*}
        \setlength{\abovedisplayskip}{1pt}
        \setlength{\belowdisplayskip}{1pt}
        s_{i,t} \in \boldsymbol{s}^{min} \! \! \iff s_{i,t} \to \mydots \to r_{t+\tau} \text{ for } \tau \geq 0  \text{, and }
        \theta_{i} \in \boldsymbol{\theta}^{min} \! \! \iff \theta_{i}  \to \mydots \to r_{t+\tau} \text{ for } \tau \geq 0
    \end{align*}

    \textbf{Continuous changes.} If $g^s$ and $g^r$ are continuous, then they can model smooth changes in the environment, including across episodes. While the functions in Eq.~\ref{eq: adarl1}-\ref{Eq: conti_model} allow us to model \emph{within-episode} changes, i.e. changes that can happen only before $t=H$ where $H$ is the horizon, we also want to model \emph{across-episode} changes. We use a separate time index $\tilde{t}$ that models the agent's lifetime. Initially $\tilde{t}=t$ for $t \leq H$, but while we reset $t=0$ afterwards, $\tilde{t}$ continues to grow indefinitely.

    \textbf{Discrete changes.}
    We assume discrete changes happen at specific timesteps and can be represented with a piecewise constant function.
    In particular we denote change timesteps as $\boldsymbol{\tilde{t}} = (\tilde{t}_1, \ldots, \tilde{t}_M)$ where $\tilde{t}_i$ describes a specific timepoint in the agent's lifetime time index $\tilde{t}$.
    This allows us to model \emph{within-episode changes}. In this case, we assume that the change happens always at the same steps $(t_1, \mydots, t_m)$ in each episode, i.e., we assume that $\boldsymbol{\tilde{t}} =(t_1, \mydots, t_m, H+t_1, \mydots, H+t_m, 2H+t_1, \mydots, 2H+t_m, \mydots)$, where $H$ is the horizon.
    We can also model \emph{across-episode changes}, when we assume the change points only occur at the end of each episode, i.e., $\boldsymbol{\tilde{t}} = (H, 2H, 3H, \mydots)$.

    We extend the result on identifiability of factored MDPs in AdaRL \citep{huang2021adarl} to non-stationary environments.
    We first assume that we observe the change factors and show we can identify the true causal graph $\mathcal{G}$:
    
    \begin{proposition}[Full identifiability with observed change factors]
        Suppose the generative process follows Eq.~\ref{eq: adarl1}-\ref{Eq: conti_model} and all change factors $\boldsymbol{\theta^s}_{t}$ and $\boldsymbol{\theta^r}_{t}$ are observed, i.e., Eq.~\ref{eq: adarl1}-\ref{Eq: conti_model} is an MDP. Under the Markov and faithfulness assumptions, i.e. conditional independences correspond exactly to d-separations, all binary masks $\boldsymbol{C}^{\cdot \veryshortarrow \cdot}$ are identifiable, i.e., we can fully recover the causal graph $\mathcal{G}$.
        \label{prop:identifibility}
    \end{proposition}
    We provide all proofs and a detailed explanation in Appendix~\ref{app:proofs}.
    If we do not observe the change factors $\boldsymbol{\theta^s}_{t}$ and $\boldsymbol{\theta^r}_{t}$, we cannot identify their dimensions, and we cannot fully recover the causal graph $\mathcal{G}$.
    On the other hand, we can still identify the partial causal graph over the state variables $\mathbf{s}_t$, reward variable $r_t$, and action variable $\mathbf{a}_t$. We can also identify which dimensions in $s_{i,t}$ have changes, i.e., we can identify $\boldsymbol{C}^{\boldsymbol{\theta}^s \scriptveryshortarrow \boldsymbol{s}}$. We formalize this idea in the following (proof in Appendix~\ref{app:proofs}):
    \begin{proposition}[Partial Identifiability with latent change factors]
        Suppose the generative process follows Eq.~\ref{eq: adarl1}-\ref{Eq: conti_model} and the change factors $\boldsymbol{\theta^s}_{t}$ and $\boldsymbol{\theta^r}_{t}$ are unobserved.
        Under the Markov and faithfulness assumptions, the binary masks $\boldsymbol{C}^{\boldsymbol{s} \scriptveryshortarrow \boldsymbol{s}}, \boldsymbol{C}^{\boldsymbol{a} \scriptveryshortarrow \boldsymbol{s}}, \boldsymbol{c}^{\boldsymbol{s} \scriptveryshortarrow r}$ and $\boldsymbol{c}^{\boldsymbol{a} \veryshortarrow r}$ are identifiable. Moreover, we can identify which state dimensions are affected by $\boldsymbol{\theta^s}_{t}$
        and whether the reward function changes.  
        \label{prop:part-identifibility}
    \end{proposition}
    This means that even in the most general case, we can learn most of the true causal graph $G$ in an FN-MDP, with the exception of the transition structure of the latent change factors.
    In the following, we show a variational autoencoder setup to learn the generative process in FN-MDPs.

    \section{Learning the Generative Process in FN-MDPs}
    \label{sec:FAEVAE}
    There are many possible architectures to learn FN-MDPs through Eq.~\ref{eq: adarl1}-\ref{Eq: conti_model}. We propose \emph{FN-VAE}, a variational autoencoder architecture described in Fig.~\ref{fig:vae}(b). In FN-VAE, we jointly learn the structural relationships, state transition function, reward function, and transition function of the latent change factors, as described in detail in Appendix Alg.~\ref{alg:A1}.
    An FN-VAE has four types of components: \emph{change factor (CF) inference networks} that reconstruct the latent change factors, \emph{change factor (CF) dynamics networks} that model their dynamics with an LSTM ~\citep{hochreiter1997long},  \emph{transition decoders} that reconstruct the state dynamics at the time $t$ and predict one step further at $t+1$, and \emph{reward decoders} that reconstruct the reward at $t$ and predict the future reward at $t+1$. We now describe them in detail.

    \textbf{CF inference networks (blue boxes in Fig.~\ref{fig:vae}(b)).} The two inference models for latent change factors $q_{\phi^s} \left(\boldsymbol{\theta^s_t} \mid \boldsymbol{s}_t, \boldsymbol{a}_t \right)$ and $q_{\phi^r} \left(\boldsymbol{\theta}_t^r \mid \boldsymbol{s}_t, \boldsymbol{a}_t, r_t \right)$ are parameterised by $\phi^s$ and $\phi^r$, respectively. 
    To model the time-dependence of $\boldsymbol{\theta_t^s}$ and $\boldsymbol{\theta}_t^r$, we use LSTMs~\citep{hochreiter1997long} as inference networks.
    At timestep $t$, the dynamics change factor LSTM infers $q_{\phi^s}\left( \boldsymbol{\theta}^s_{t} \mid \boldsymbol{s}_t, \boldsymbol{a}_t, r_t, \boldsymbol{h}^s_{t-1} \right)$, where $\boldsymbol{h}^s_{t-1} \in \mathbb{R}^L$ is the hidden state in the LSTM.
    Thus we can obtain $\mu_{\phi^s}(\boldsymbol{\tau}_{0:t})$ and $\sigma^2_{\phi^s}(\boldsymbol{\tau}_{0:t})$ using $q_{\phi^s}$, and sample the latent changing factor $\boldsymbol{\theta}_{t}^{\boldsymbol{s}} \sim \mathcal{N}(\mu_{\phi^s}(\boldsymbol{\tau}_{0:t}), \sigma_{\phi^s}^{2}(\boldsymbol{\tau}_{0:t}))$, where $\boldsymbol{\tau}_{0:t}=\left(\boldsymbol{s}_0, \boldsymbol{a}_0, r_0, \boldsymbol{s}_1, \boldsymbol{a}_1, r_2, \ldots, \boldsymbol{s}_t, \boldsymbol{a}_t, r_{t} \right)$.
    Similarly, the reward change factor LSTM infers $q_{\phi^r}( \boldsymbol{\theta}^r_{t} \mid \boldsymbol{s}_t, \boldsymbol{a}_t, r_t, \boldsymbol{h}^r_{t-1})$, where $\boldsymbol{h}^r_{t-1} \in \mathbb{R}^L$ is the hidden state, such that we can sample ${\boldsymbol{\theta}}_{t}^r \sim \mathcal{N}(\mu_{\phi^r}(\boldsymbol{\tau}_{0:t}), \sigma_{\phi^r}^{2}(\boldsymbol{\tau}_{0:t}))$.

    \textbf{CF dynamics network (orange boxes in Fig.~\ref{fig:vae}(b)).}
    We model the dynamics of latent change factors with $p_{\gamma^s} \left (\boldsymbol{\theta^s}_{t+1} \mid \boldsymbol{\theta^s}_{t}, \boldsymbol{C}^{\boldsymbol{\theta^s} \scriptveryshortarrow \boldsymbol{\theta^s}}\right)$ and $p_{\gamma^r} \left (\boldsymbol{\theta^r}_{t+1} \mid \boldsymbol{\theta^r}_{t}, \boldsymbol{C}^{\boldsymbol{\theta^r} \scriptveryshortarrow \boldsymbol{\theta^r}}\right)$. To ensure the Markovianity of $\boldsymbol{\theta}^s_t$ and $\boldsymbol{\theta}^r_t$, we define a loss $\mathcal{L}_{\text{KL}}$ that helps minimize the KL-divergence between $q_\phi$ and $p_\gamma$.
    \begin{equation}
        \label{eq:kl}
        \begin{aligned}
            \mathcal{L}_{\text{KL}} = \sum \limits_{t=2}^T & \text{KL}  \big( q_{\phi^s}\left( \boldsymbol{\theta}^s_{t} \mid \boldsymbol{s}_t, \boldsymbol{a}_t, r_t, \boldsymbol{h}^s_{t-1} \right)) \Vert   {p_{\gamma^s}(\boldsymbol{\theta^s}_{t}|\boldsymbol{\theta^s}_{t-1}; \boldsymbol{C}^{\boldsymbol{\theta^s} \scriptveryshortarrow \boldsymbol{\theta^s}})} \big) \\
            + & \text{KL}  \big( q_{\phi^r}( \boldsymbol{\theta}^r_{t} \mid \boldsymbol{s}_t, \boldsymbol{a}_t, r_t, \boldsymbol{h}^r_{t-1})) \Vert   {p_{\gamma^r}(\boldsymbol{\theta}^r_{t}|\boldsymbol{\theta}^r_{t-1}; \boldsymbol{C}^{\boldsymbol{\theta}^r \scriptveryshortarrow \boldsymbol{\theta}^r})} \big)
        \end{aligned}
    \end{equation}
    
    If we assume that the change between $\boldsymbol{\theta^s}_{t}$ and $\boldsymbol{\theta^s}_{t+1}$, and similarly $\boldsymbol{\theta^r}_{t}$, is smooth, we can add a smoothness loss $\mathcal{L}_{\text{smooth}}$. We provide a smooth loss for discrete changes in Appendix~\ref{app: discrete}.

    \begin{equation} 
        \setlength{\abovedisplayskip}{1pt}
        \setlength{\belowdisplayskip}{1pt}
        \mathcal{L}_{\text{smooth}} =  \sum
        \limits_{t=2}^T \left(|| \boldsymbol{\theta^s}_{t} - \boldsymbol{\theta^s}_{t-1}||_1 +  ||\boldsymbol{\theta^r}_{t} - \boldsymbol{\theta^r}_{t-1}||_1  \right)
    \end{equation}

    \textbf{Transition decoders (purple boxes in Fig.~\ref{fig:vae}(b)).} We learn an approximation of the transition dynamics in Eq.~\ref{eq: adarl1} by learning a reconstruction, parameterized by $\alpha_1$, and a prediction encoder, parametrized by $\alpha_2$. To simplify the formulas, we define $\boldsymbol{C}^{\cdot \scriptveryshortarrow \boldsymbol{s}} := (\boldsymbol{C}^{\boldsymbol{s} \scriptveryshortarrow \boldsymbol{s}}, \boldsymbol{C}^{\boldsymbol{a} \scriptveryshortarrow \boldsymbol{s}}, \boldsymbol{C}^{\boldsymbol{\theta^s} \scriptveryshortarrow \boldsymbol{s}})$.
    At timestep $t$, the reconstruction encoder $p_{\alpha_1}\left(\boldsymbol{s}_t \mid \boldsymbol{s}_{t-1}, \boldsymbol{a}_{t-1}, \boldsymbol{\theta}^s_t; \boldsymbol{C}^{\cdot \scriptveryshortarrow \boldsymbol{s}} \right)$  reconstructs the state from current state $\boldsymbol{s}_t$ with sampled $\boldsymbol{\theta^s}_t$.
    The one-step prediction encoder $p_{\alpha_2}\left(\boldsymbol{s}_{t+1} \mid \boldsymbol{s}_{t}, \boldsymbol{a}_{t}, \boldsymbol{\theta^s}_t \right)$ instead tries to approximate the next state $\boldsymbol{s}_{t+1}$. We do not use the prediction loss, when the one-step prediction is not smooth.
    In particular, we do not use it for the last time-step in episode $i$ if there is a change happening at the first step in episode $(i+1)$, since the states in new episodes will be randomly initiated. We also do not use it in the case of discrete changes at the timesteps $(\tilde{t}_1-1, \ldots, \tilde{t}_M-1)$. The loss functions are:
    \begin{equation}
        \begin{array}{l}
            \mathcal{L}_{\text{rec-dyn}} =
            \sum \limits_{t=1}^{T-2} \mathbb{E}_{\theta^s_t \sim q_{\phi}} \log p_{\alpha_1}(\boldsymbol{s}_{t} |\boldsymbol{s}_{t-1}, \boldsymbol{a}_{t-1},\boldsymbol{\theta}_{t}^s; \boldsymbol{C}^{\cdot \scriptveryshortarrow \boldsymbol{s}}) 
            \\
            \mathcal{L}_{\text{pred-dyn}} =  \sum \limits_{t=1}^{T-2} \mathbb{E}_{\theta^s_t \sim q_{\phi}} \log p_{\alpha_2}(\boldsymbol{s}_{t+1}| \boldsymbol{s}_{t},
            \boldsymbol{a}_{t},\boldsymbol{\theta}_{t}^s)
            \label{eq:dyn}
        \end{array}
    \end{equation}

    \textbf{Reward decoders (green boxes in Fig.~\ref{fig:vae}(b)).} Similarly, we use a reconstruction encoder $p_{\beta_1}\left(r_t \mid \boldsymbol{s}_t, \boldsymbol{a}_t, \boldsymbol{\theta}^r_t, \boldsymbol{c}^{\boldsymbol{s} \scriptveryshortarrow r}, \boldsymbol{c}^{\boldsymbol{a} \scriptveryshortarrow r} \right)$, parameterized by $\beta_1$, and a one-step prediction encoder $p_{\beta_2}\left(r_{t+1} \mid \boldsymbol{s}_{t+1}, \boldsymbol{a}_{t+1}, \boldsymbol{\theta}^r_t \right)$, parametrized by $\beta_2$, to approximate the reward function. Similarly to transition decoders, we do not use the one-step prediction loss, if it is not smooth. The losses are:
    \begin{equation}
        \begin{array}{l}
            \mathcal{L}_{\text{rec-rw}}  = \sum \limits_{t=1}^{T-2} \mathbb{E}_{\theta^r_t \sim q_{\phi}} \log p_{\beta_1}(r_{t} |\boldsymbol{s}_{t}, \boldsymbol{a}_{t},\boldsymbol{\theta}_{t}^r; \boldsymbol{c}^{\boldsymbol{s} \scriptveryshortarrow r}, \boldsymbol{c}^{\boldsymbol{a} \scriptveryshortarrow r})
            \\
            \mathcal{L}_{\text{pred-rw}} = \sum \limits_{t=1}^{T-2} \mathbb{E}_{\theta^r_t \sim q_{\phi}} \log p_{\beta_2}(r_{t+1}| \boldsymbol{s}_{t+1}, \boldsymbol{a}_{t+1},\boldsymbol{\theta}_{t}^r)
            \label{eq:rw}
        \end{array}
    \end{equation}
    
    \textbf{Sparsity loss.} We encourage sparsity in the binary masks $\boldsymbol{C}^{\cdot \veryshortarrow \cdot}$ to improve identifiability, by using following loss with adjustable hyperparameters $(w_1, \mydots, w_7)$, which we learn through grid search.
    
    \begin{equation}
        \begin{aligned}
            \mathcal{L}_{\text{sparse}} &= w_1\| \boldsymbol{C}^{\boldsymbol{s} \scriptveryshortarrow \boldsymbol{s}} \|_1 + w_2\| \boldsymbol{C}^{\boldsymbol{a} \scriptveryshortarrow \boldsymbol{s}} \|_1 + w_3\| \boldsymbol{C}^{\boldsymbol{\theta^s} \scriptveryshortarrow \boldsymbol{s}} \|_1
            + w_4\| \boldsymbol{c}^{\boldsymbol{s} \scriptveryshortarrow r} \|_1 + w_5\| \boldsymbol{c}^{\boldsymbol{a} \scriptveryshortarrow r} \|_1 \\
            & + w_6\| \boldsymbol{C}^{\boldsymbol{\theta^s} \scriptveryshortarrow \boldsymbol{\theta^s}} \|_1 + w_7\| \boldsymbol{C}^{\boldsymbol{\theta}^r \scriptveryshortarrow \boldsymbol{\theta}^r}\|_1
        \end{aligned}
        \label{eq:sparse}
    \end{equation}
    
    The total loss is $\mathcal{L}_{\text{vae}} = k_1(\mathcal{L}_{\text{rec-dyn}} +
    \mathcal{L}_{\text{rec-rw}})+
    k_2(\mathcal{L}_{\text{pred-dyn}} + \mathcal{L}_{\text{pred-rw}})
    - k_3\mathcal{L}_{\text{KL}} - k_4\mathcal{L}_{\text{sparse}} - k_5\mathcal{L}_{\text{smooth}}$, where $(k_1, \mydots, k_5)$: hyper-parameters, which we learn with an automatic weighting method  \citep{liebel2018auxiliary}.

    \textbf{Learning from raw pixels.}
    Our framework can be easily extended to image inputs by adding an encoder $\phi^o$ to learn the latent state variables from pixels, similar to other works ~\citep{ha2018recurrent, meld, huang2021adarl}. In this case, our identifiability results do not hold anymore, since we cannot guarantee that we identify the true causal variables. We describe this component in Appendix~\ref{app: pixel}.

    \begin{algorithm}[!h]
        \caption{Factored Adaptation for non-stationary RL}
        \label{alg:A2}
        \begin{algorithmic}[1]
            \STATE {\bfseries Init:} Env; VAE parameters: $\phi = (\phi^s, \phi^r)$, $\alpha = (\alpha_1, \alpha_2)$, $\beta = (\beta_1, \beta_2)$, $\gamma$; Binary masks: $\boldsymbol{C^{\cdot \scriptveryshortarrow \cdot}}$; Policy parameters: $\psi$; replay buffer: $\mathcal{D}$; Number of episodes: $N$; Episode horizon: $H$; Initial $\boldsymbol{\theta}$: $\boldsymbol{\theta}^s_{\text{old}}$ and  $\boldsymbol{\theta}^r_{\text{old}}$; Length of collected trajectory: $k$.
            \STATE {\bfseries Output:} VAE parameters: $\phi$, $\alpha$, $\beta$, $\gamma$; Policy parameters: $\psi$
            \STATE Collect multiple trajectories of length $k$ : $\boldsymbol{\tau} = \{\boldsymbol{\tau}^{1}_{0:k}, \boldsymbol{\tau}^{2}_{0:k}, \ldots\}$ with policy $\pi_{\psi}$ from Env;
            \STATE Learn FN-VAE from $\boldsymbol{\tau}$, including masks $\boldsymbol{C}^{\cdot \scriptveryshortarrow \cdot}$ that represent the graph $\mathcal{G}$ (Appendix Alg.~\ref{alg:A1})
            \STATE Identify the compact representations $\boldsymbol{s}^{min}$ and change factors $\boldsymbol{\theta}^{min}$ based on $\boldsymbol{C}^{\cdot \scriptveryshortarrow \cdot}$
            \FOR {$n=0,\ldots, N-1$}
            \FOR{$t=0, \ldots, H-1$}
            \STATE Observe $\boldsymbol{s}_{t}$ from Env;
            \IF{$t=0$}
            \STATE $\boldsymbol{\theta}^s \leftarrow \boldsymbol{\theta}^s_{\text{old}}$ and $\boldsymbol{\theta}^r \leftarrow  \boldsymbol{\theta}^r_{\text{old}}$
            \ELSE
            \STATE $\boldsymbol{\theta}^s \leftarrow \boldsymbol{\theta}^s_{t-1}$ and $\boldsymbol{\theta}^r \leftarrow  \boldsymbol{\theta}^r_{t-1}$
            \ENDIF
            \FOR {j = s, r }
            \STATE Infer mean $\mu_{\gamma^j}(\boldsymbol{\theta}^j)$ and variance $\sigma^2_{\gamma^j}(\boldsymbol{\theta}^j)$ of the change parameter $\boldsymbol{\theta}^j_{t}$ via $p_{\gamma^j}$
            \STATE Sample ${\boldsymbol{\theta}}^j_{
                t} \sim \mathcal{N}\left(\mu_{\gamma^j}(\boldsymbol{\theta}^j), \sigma_{\gamma^j}^{2}(\boldsymbol{\theta}^j)\right)$
            \ENDFOR

            \IF{$t=H-1$}
            \STATE $\boldsymbol{\theta}^s_{\text{old}} \leftarrow$  $\boldsymbol{\theta}^s_{t}$ and $\boldsymbol{\theta}^r_{\text{old}} \leftarrow$  $\boldsymbol{\theta}^s_{t}$;
            \ENDIF
            \STATE Generate $\boldsymbol{a}_{t} \sim \pi_\psi (\boldsymbol{a}_{t} \mid \boldsymbol{s}^{min}_{t}, \boldsymbol{\theta}^{min}_{t})$  and receive reward $r_{n,t}$ from Env;
            \STATE Add $(\boldsymbol{s}_{t}, \boldsymbol{a}_{t}, r_t, \boldsymbol{\theta}^s_{t}, \boldsymbol{\theta}^r_{t})$ to replay buffer $\mathcal{D}$;
            \STATE  Extract a trajectory with length $k$ from replay buffer $\mathcal{D}$;
            \STATE Learn FN-VAE (Appendix Alg.~\ref{alg:A1}) with updateG=False (i.e. with fixed masks $\boldsymbol{C}^{\cdot \scriptveryshortarrow \cdot}$);
            \STATE Sample a batch of data from replay buffer $\mathcal{D}$ and update policy network parameters $\psi$;
            \ENDFOR
            \ENDFOR
        \end{algorithmic}
    \end{algorithm}

    \section{FANS-RL: Online Model Estimation and Policy Optimization}
    We propose Factored Adaptation for Non-Stationary RL (FANS-RL), a general algorithm that interleaves model estimation and policy optimization, as shown in Alg.~\ref{alg:A2}.
    After we estimate the initial FN-MDP with the FN-VAE,
    we can identify compact representations $\boldsymbol{s}^{min}$ and $\boldsymbol{\theta}^{min}$ following AdaRL~\citep{huang2021adarl}. In particular, the only dimensions of the state and change factors that are useful for policy learning are those that have a directed path to the reward in the graph $\mathcal{G}$.
    The online policy $\pi_{\psi}\left(\boldsymbol{a}_{t} \mid \boldsymbol{s}_{t}^{min}, q_{\phi}(\boldsymbol{\theta}_t^{min} \mid \boldsymbol{\tau}_{0:t})\right)$ can be learned end-to-end, including learning the FN-VAE, as shown in Alg.~\ref{alg:A2}. We use SAC~\citep{haarnoja2018soft} as our policy learning model, so the policy parameters are $\psi = (\pi, Q)$.

    \textbf{Continuous changes.} In Alg.~\ref{alg:A2} we describe our framework in case of continuous changes that can span across episodes.
    We start by collecting a few trajectories $\tau$ and then learn our initial FN-VAE (Lines 3-4). We can use the graphical structure of the initial FN-VAE to identify the compact representations (Line 5).
    During the online model estimation and policy learning stage, we estimate the latent change factors $\boldsymbol{\theta}^s_{t}$ and $\boldsymbol{\theta}^r_{t}$ using the CF dynamics networks $\gamma^s$ and $\gamma^r$ (Lines 8-20). Since in this case, we assume the dynamics of the change factors are smooth across episodes, at time $t=0$ we will use the last timestep $(H-1)$ of the previous episode as a prior on the change factors (Line 10). Otherwise, we will estimate the change factors using their values in the previous timestep $t-1$.
    We use the estimated latent factors $\boldsymbol{\theta}_{t}$ and observed state $\boldsymbol{s}_{t}$ to generate $\boldsymbol{a}_{t}$ using $\pi_\psi$ and receive a reward $r_{t}$ (Line 21). We add $(\boldsymbol{s}_{t}, \boldsymbol{a}_{t}, r_t, \boldsymbol{\theta}^s_{t}, \boldsymbol{\theta}^r_{t})$ to the replay buffer (Line 22). We now update our estimation of the FN-VAE, but we keep the graph $G$ fixed (Lines 23-24). Finally, we sample a batch of trajectories in the replay buffer and update the policy network $\psi$ (Line 25).

    \textbf{Discrete changes.}
    Since we assume discrete changes happen at specific timestep $\boldsymbol{\tilde{t}} = (\tilde{t}_1, \ldots, \tilde{t}_M)$, we can easily modify Alg.~\ref{alg:A2} for discrete changes, both within-episode and across-episode, by changing Lines 9-20 to only update the change parameters at the timesteps in $\mathbf{\tilde{t}}$, as shown in Appendix Alg.~\ref{alg:A3}.

    \section{Evaluation}
    \label{sec:eval}
    \begin{figure*}[t]
        \centering
        \includegraphics[width=1\linewidth]{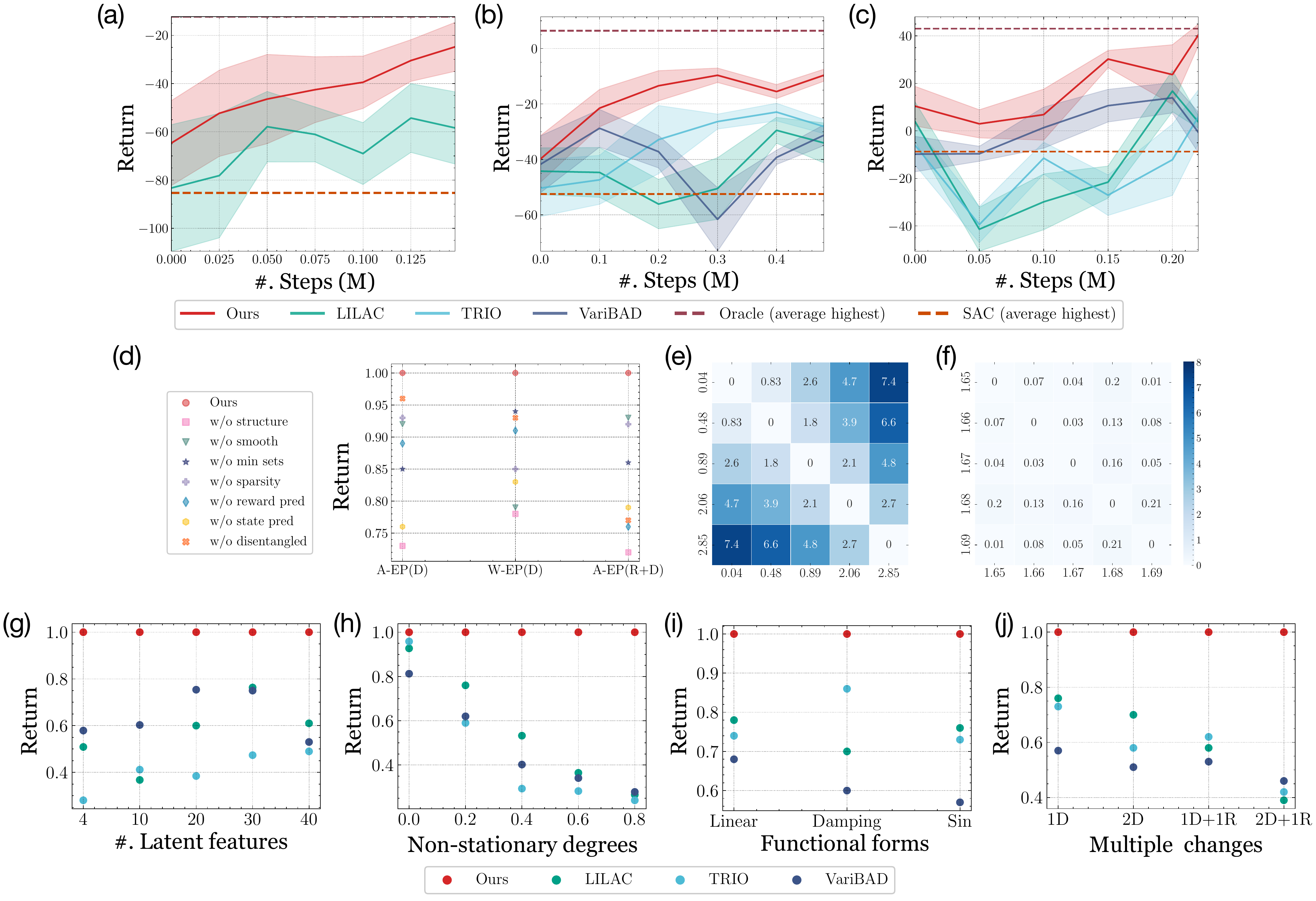}
        \caption{
            Summary of experimental results. (a)-(c). Average return (smoothed) across 10 runs. We only indicate the average of the highest result of all times for oracle and SAC. The shaded region is ($\mu - \sigma, \mu + \sigma$), where $\mu$ is the mean and $\sigma$ is the standard deviation. (a) Half-Cheetah-V3 with continuous (sine) changes on $f_w$; (b)  Sawyer-Reaching with discrete across-episode changes on $\boldsymbol{s}^g$; and (c) Minitaur with discrete across-episode changes on $m$ and $s_{t,v}$ concurrently. (d) Ablation studies on Half-Cheetah with across \& within episode changes on dynamics and across episode changes on both dynamics and rewards. (e)-(f): Pairwise distance on learned $\boldsymbol{\theta}$ between different time steps in Half-Cheetah experiment with across-episode changes on rewards. (g)-(j): Average and normalized final return on 10 runs on Half-Cheetah (g) with within-in episode changes on wind forces using a different number of dimensions in latent representation space; (h) with different non-stationary degrees on across-episode and multi-factor changes; (i) with different functional forms (across episode changes on dynamics); and (j) with different combinations of across-episode changes.}
        \label{fig:reward_cheetah}
    \end{figure*}

    We evaluate our approach on four well-established benchmarks, including Half-Cheetah-V3 from MuJoCo \citep{todorov2012mujoco, brockman2016openai}, Sawyer-Reaching and Sawyer-Peg from Sawyer \citep{yu2020meta, meld}, and Minitaur \citep{tan2018sim}. We modified these tasks to test several non-stationary RL scenarios with continuous and discrete changes.
    The results suggest that FANS-RL can (1) obtain high rewards, (2) learn meaningful mappings to the change factors with compact latent vectors, and (3) be robust to different non-stationary levels, different functional forms (e.g. piecewise linear, sine, damped sine) and multiple concurrent changes. For space limits, we only highlight a subset of results, and report the full results in Appendix~\ref{app:full}. The implementation will be open-sourced at \url{https://bit.ly/3erKoWm}.

    \textbf{Half-Cheetah-v3 in MuJoCo.} 
    In this task, the agent is moving forward using the joint legs and the objective is to achieve the target velocity ${v}^g$.
    We consider both the changes in the dynamics (change of the wind forces $f^w$, change of gravity) and reward functions (change of target velocity ${v}^g$). We also consider changes on the agent's mechanism, where one random joint is disabled. The reward function is $r_t=-\left\|v^{o}_t-v^{g}_t\right\|_{2}-0.05\|{a}_t\|_{2}$, where $v^{o}$ and $a_t$ are the agent's velocity and action, respectively, at timestep $t$. The number of time steps in each episode is $50$. For dynamics, we change the wind forces $f^w$ in the environment. Moreover, in terms of the reward functions, we change the target velocity ${v}^g$ to be a time-dependent variable. 
    The change function in the dynamics $f^w$ can be either continuous or discrete, and discrete changes can happen both within and across episodes. Similarly to LILAC \citep{xie2020deep}, we choose different functions (piecewise linear, sine and damped sine), besides allowing the change at specified intervals, we also allow it to change smoothly.
    The change in the reward function $v^g$ is not generally stable in the continuous case, so we only consider the discrete and across episode change functions $v^g$ for the reward.
    We also design a scenario where dynamic and reward functions change concurrently.
    We report all equations for $g^s$ and $g^r$ in Appendix~\ref{app:mujoco}.

    \textbf{Sawyer.}
    We consider two robotic manipulation tasks, Sawyer-Reaching and Sawyer-Peg. We describe the non-stationary settings in Appendix~\ref{app:sawyer}.
    
    In Sawyer-Reaching, the sawyer arm is trained to reach a target position $\boldsymbol{s}^{g}$.
    The reward $r_t$ is the difference between the current position $\boldsymbol{s}_t$ and the target position $r_t=-\left\|\boldsymbol{s}_t-\boldsymbol{s}^{g}\right\|_{2}$. In this task, we cannot directly modify the dynamics in the simulator, so consider a reward-varying scenario, where the target location changes across each episode following a periodic function.

    In Sawyer-Peg, the robot arm is trained to insert a peg into a designed target location $\boldsymbol{s}^g$. In this task, following~\citep{sodhani2021block}, we consider a reward-varying scenario, where the target location changes across each episode following a periodic function. In order to compare with similar approaches, e.g., ZeUS~\citep{sodhani2021block}, CADM~\citep{lee2020context}, Hyperdynamics~\citep{xian2021hyperdynamics}and Meld~\citep{meld}, we evaluate our method on raw pixels. Following ~\citep{sodhani2021block}, we consider discrete across-episode changes, where the target location can change in each episode, and is randomly sampled from a small interval. 

    \textbf{Minitaur.} A minitaur robot is a simulated quadruped robot with eight direct-drive actuators. The minitaur is trained to move at the target speed  $s_{t,v}$. The reward is $r_t=0.3-\left|0.3-\boldsymbol{s}_{t, v}\right|-0.01 \cdot\left\|\boldsymbol{a}_{t}-2 \boldsymbol{a}_{t-1}+\boldsymbol{a}_{t-2}\right\|_{1}$.
    We modify (1) the mass $m$ (dynamics) and (2) target speed $s_{t,v}$ (reward) of the minitaur. We consider continuous, discrete across-episode and discrete within-episode changes for the dynamics, and across-episode changes for the reward. We describe the settings in Appendix~\ref{app:minitaur}.

    \textbf{Baselines.}
    We compare our approach with a meta-RL approach for stationary RL, \textbf{VariBAD}~\citep{whiteson2021varibad}, a meta-RL approach for non-stationary RL, \textbf{TRIO}~\citep{ijcai2021-399}, as well as with two representative task embedding approaches, \textbf{LILAC}~\citep{xie2020deep} and \textbf{ZeUS}~\citep{sodhani2021block}.
    The details on the meta-learning setups are given in Appendix~\ref{app:hyper}.
    We also compare with stationary RL method, \textbf{SAC}~\citep{haarnoja2018soft}, which will be our lower-bound, and compare with an \textbf{oracle} agent that has access to the full information of non-stationarity (e.g., the wind forces) and can use it to learn a policy, which will be our upper-bound.
    For all baselines, we use SAC for policy learning.
    We compare with the baselines on \textit{average return} and \textit{the compactness of the latent space} in varying degrees and functional forms of non-stationarity.

    \textbf{Experimental results.}
    Fig.~\ref{fig:reward_cheetah}\textcolor{RoyalPurple}{(a)-(c)} shows the smoothed curves of average return across timesteps in a subsection of Half-Cheetah, Sawyer-Reaching, and Minitaur experiments.
    Smoothed curves of other experiments are given in Appendix Fig.~\ref{fig:appendix_cheetah_smooth} and \ref{fig:appendix_minitaur_smooth}.
    For continuous changes, we only compare with LILAC~\citep{xie2020deep} since other approaches are not applicable. We smooth the learning curves by uniformly selecting a few data points for readability. 
    \begin{table}[t]
        \small
        \caption{Average highest return on Sawyer-Peg. The number of random trails is $10$. }
        \begin{tabular}{c|c|c|c|c|c}
            \toprule
            Methods & ZeUS    & Meld    & CaDM   & Hyper-Dynamics & Ours             \\ \hline
            Best Avg. Return   & $12.45$ & $6.38$  & $4.04$ & $3.91$         & $\textbf{18.01}$ \\ 
            \bottomrule
        \end{tabular}
        \label{tbl:pomdp}
    \end{table}
    Table~\ref{tbl:pomdp} shows the results on Sawyer-Peg experiments using raw pixels as input, based on the reported results in~\citep{sodhani2021block}. 
    The learning curve is given in Appendix Fig.~\ref{fig:pomdp}. For a fair comparison, we indicate the best and final average return for each baseline in Fig.~\ref{fig:pomdp}. 
    Full quantitative results of all experiemnts are given in Appendix~\ref{app:full}, including significance tests for Wilcoxon signed-rank test at $\alpha=0.05$ showing that FANS-RL is significantly better than baselines. Finally, we conduct ablation studies on each component of FANS-RL and report some key results in Fig.~\ref{fig:reward_cheetah}\textcolor{RoyalPurple}{(d)}. Full results are in Appendix~\ref{app:full}. 

    \textbf{Ablation studies.} The ablation studies verify the effectiveness of all components, including binary masks/structure, smoothness loss, sparsity loss, reward prediction or state prediction. The results show that \textit{the largest gain is provided by the factored representation}, validating our original hypothesis, followed by state prediction. As expected, reward prediction is also important when there is a nonstationary reward, while smoothness is important for within-episode changes. The disentangled design of CF inference networks is valuable when there are changes on both dynamics and reward functions. Full results are in Appendix~\ref{app:ablation}.

    \textbf{Visualization on the learned $\boldsymbol{\theta}$.}
    To verify that the learned $\boldsymbol{\theta}$ can capture the true change factors, we compute the pairwise distance between learned $\boldsymbol{\theta}^r$ at different time steps. We randomly sample $10$ time steps from the Half-Cheetah experiment with across-episode changes on reward functions.
    Fig.~\ref{fig:reward_cheetah}\textcolor{RoyalPurple}{(e)} gives the pairwise distance of $\boldsymbol{\theta}^r$ among $5$ time steps (from episode $148, 279, 155, 159, 230$) with different target speed values ($0.04, 0.48, 0.89, 2.06, 2.85$), respectively. We can find that there is a positive correlation between the distance of learned $\boldsymbol{\theta}^r$ and values of change factors. Meanwhile, we also sample 5 time steps from episodes $31, 188, 234, 408$ with target values around $1.67$. Fig.~\ref{fig:reward_cheetah}\textcolor{RoyalPurple}{(f)} shows that the distance of $\boldsymbol{\theta}^r$ among these $5$ time steps is very small, indicating that the learned $\boldsymbol{\theta}^r$ are almost the same for the similar values of change factors at different time steps. The visualization suggests that the learned $\boldsymbol{\theta}$ can capture meaningful mappings from the time-varying factors.

    \textbf{Varying latent dimensions, non-stationary levels and functional forms.} Fig.~\ref{fig:reward_cheetah}\textcolor{RoyalPurple}{(g)} shows the normalized averaged return versus the number of two latent features on Half-Cheetah. Our framework can learn a better policy with relatively smaller feature dimensions in the latent space than other approaches. As we show in Appendix~\ref{app:full}, Saywer and Minitaur have the similar trend, where we learn a better policy than the baselines with fewer latent features.
    We also vary the non-stationary levels in Half-Cheetah with discrete across-episode changes on both dynamics and rewards. A higher non-stationary degree indicates a faster change rate. Fig.~\ref{fig:reward_cheetah}\textcolor{RoyalPurple}{(h)} shows that FANS-RL achieves the highest return across all tested non-stationary degrees and that the gap increases with the non-stationarity.
    We also test FANS-RL together with all baselines on different non-stationary functions, including piecewise linear, damping-like and sinusoid waves. Fig.~\ref{fig:reward_cheetah}\textcolor{RoyalPurple}{(i)} displays the results, which indicate that FANS-RL can generally outperform the baselines on diverse non-stationary function forms. Detailed function equations and experimental setups can be referred to Appendix~\ref{appendix_design}.

    \textbf{Multiple change factors} We consider different numbers and types of changes to verify the benefits from the factored structure in FANS-RL. We conduct experiments with 1) only change wind forces (1D); 2) change wind forces and gravity concurrently (2D); 3) change wind force and target speed (1D+1R); and 4) change wind force, gravity, and target speed together (2D+1R) in an across-episode way in Half-Cheetah. From Fig.~\ref{fig:reward_cheetah}\textcolor{RoyalPurple}{(j)}, we find that, thanks to the factored representation, FANS-RL performs better in those more complicated scenarios with multiple numbers and types of changes.

    \section{Related Work}
     \textbf{Non-stationary RL.}
    Early works in non-stationary RL ~\citep{da2006dealing,sutton2007role} only detect changes that have already happened instead of anticipating them.
    If the evolution of non-stationary environments is a (Semi)-Markov chain, one can deal with non-stationarity with HM-MDPs~\citep{choi1999environment} or HS3MDPs~\citep{hadoux2014solving}.
    Several methods learn to anticipate changes in non-stationary deep RL. \citet{chandak2020optimizing} propose to maximize future rewards without explicitly modeling non-stationary environments. MBCD~\citep{2021aamas} uses change-point detection to decide if the agent should learn a novel policy or reuse previously trained policies.
    \citet{alshedivat2018continuous} extend  
    MAML~\citep{finn2017model}
    for the non-stationary setting, but do not explicitly model the temporal changes.
    TRIO~\citep{ijcai2021-399} tracks the non-stationarity by inferring the evolution of latent parameters, which captures the temporal change factors during the meta-testing phase. ReBAL and GrBAL~\citep{clavera2018learning} meta-train the dynamic prior, which adapts to the local contexts efficiently. However, these methods have to meta-train the model on a set of non-stationary tasks, which may not be accessible in real-world applications.
    Another line of research directly learns the latent representation to capture the non-stationary components. In particular, LILAC~\citep{xie2020deep} and ZeUS \cite{sodhani2021block} leverage latent variable models to directly model the change factors in environments, and \citet{RL_Guo22} estimate latent vectors that describe the non-stationary or variable part of the dynamics.
    Our approach fits in this line of work; however, as opposed to these methods, which model all temporal changes using a shared embedding space or a mixture models of dynamics, we use a factored representation.

    \textbf{Factored MDPs.}
    Among works on factored representations in stationary settings, \citet{hallak2015off} learn factored MDPs ~\citep{boutilier2000stochastic, kearns1999efficient, osband2014near} to improve sample efficiency in model-based off-policy RL. \citet{balaji2020factoredrl} employ known factored MDP to improve both model-free and model-based RL algorithms. Working memory graphs~\citep{loynd2020working} learn the factored observation space using Transformers. NeverNet~\citep{wang2018nervenet} factorizes the state-action space through graph neural networks.
    \citet{zholus2022factorized} factorize the visual states into actor-related, object of manipulation, and the latent
    influence factor between these two states to achieve sparse interaction in robotic control tasks.
Differently from methods modeling the factored dynamics and rewards only, \citet{zhou2022policy} and \citet{tang2022leveraging} explore factored entities and actions, respectively.
    \citet{zhou2022policy} extend the factored MDP to the multi-entity environments, learning the compositional structure in tasks with multiple entities involved. \citet{tang2022leveraging} leverage the factored action space to improve the sample efficiency in healthcare applications. However, the factored structures in these two works are derived from inductive bias or domain experts. 
    
    \textbf{Factored MDPs, causality and multiple environments.} Several works leverage factored MDPs to improve the sample efficiency and generalization of RL. Most of these works focus on learning an invariant (e.g. causal) representation that fits all environments and do not support learning latent change factors. 
    For example, AFaR~\citep{sodhani2022improving} learns factored value functions via attention modules to improve sample efficiency and generalization across tasks.
    \citet{mutti2022provably} learn a causal structure that can generalize across a family of MDPs under different environments, assuming that there are no latent causal factors. Similarly, \citet{wang22ae} propose to learn the factored and causal dynamics in model-based RL, in which the learned causal structure is assumed to generalize to unseen states. By deriving the state abstraction based on the causal graph, it can improve both the sample efficiency and generalizability of policy learning of MBRL. While most of the previous works focus on learning a causal structure that is time-invariant, \citet{pitis2022mocoda} learn a locally causal dynamics that can vary at each timestep and use it to generate counterfactual dynamics transitions in RL.
    While the previously described methods focus on a domain or time-invariant representations, in our work we also focus on modelling domain or time-specific factors in the form of latent change factors. A related work, AdaRL~\citep{huang2021adarl} learns the factored representation and the model change factors under heterogeneous domains with varying dynamics or reward functions. However, AdaRL is designed only for the domain adaptation setting and constant change factors without considering non-stationarity.

    \textbf{Independent causal mechanisms.}
    Another related line of work is based on independent causal mechanisms~\citep{scholkopf2021toward, parascandolo2018learning}. Recurrent independent mechanisms (RIMs) \citep{goyal2021recurrent} learn the independent transition dynamics in RL with sparse communication among the latent states. Meta-RIMs~\citep{madan2021fast} leverage meta-learning and soft attention to learn a set of RIMs with competition and communication.
    As opposed to these works, we do not assume that the mechanisms are independent and we learn the factored structure among all components in MDPs with a DBN.

    \section{Conclusions, Limitations and Future Work}
    \label{sec:limit}
    We describe Factored Adaptation for Non-Stationary RL (FANS-RL), a framework that learns a factored representation for non-stationarity that can be combined with any RL algorithm.
    We formalize our problem as a Factored Non-stationary MDP (FN-MDP), augmenting a factored MDP with latent change factors evolving as a Markov process. FN-MDPs do not model a family of MDPs, but instead include the dynamics of change factors analogously to the dynamics of the states. This allows us to capture different non-stationarities, e.g., continuous and discrete changes, both within and across different episodes.
    To learn FN-MDPs we propose FN-VAEs, which we integrated in FANS-RL, an online model estimation and policy evaluation approach.
    We evaluate FANS-RL on benchmarks for continuous control and robotic manipulation, also with pixel inputs, and show it outperforms the state of the art on rewards and robustness to varying degrees of non-stationarity.
    Learning the graph in model estimation is computationally expensive, which limits the scalability of our approach. In future work, we plan to meta-learn the graphs among different tasks to improve the scalability of our approach and its applicability to complex RL problems, e.g., multi-agent RL.

    \subsection*{Acknowledgments}
    FF would like to acknowledge the CityU High-Performance Computing (HPC) resources in Hong Kong SAR and LISA HPC from the SURF.nl. SM was supported by the MIT-IBM Watson AI Lab and the Air Force Office of Scientific Research under award number FA8655-22-1-7155. BH would like to acknowledge the support of Apple Scholarship. 
    KZ was partially supported by the National Institutes of Health (NIH) under Contract R01HL159805, by the NSF-Convergence Accelerator Track-D award \#2134901, by a grant from Apple Inc., and by a grant from KDDI Research Inc.

    \newpage
    \bibliographystyle{unsrtnat}
    \bibliography{ref}

    \newpage
    \appendix
    \onecolumn
    \section*{Checklist}

    \begin{enumerate}

        \item For all authors...
        \begin{enumerate}
            \item Do the main claims made in the abstract and introduction accurately reflect the paper's contributions and scope?
            \answerYes{There are theoretical claims backed by two theoretical results in Prop.~\ref{prop:identifibility} and Prop.~\ref{prop:part-identifibility}. The empirical claims are backed by an extensive evaluation in Section~\ref{sec:eval}.}
            \item Did you describe the limitations of your work?
            \answerYes{We did provide a discussion on the limitations and potential future work in Section~\ref{sec:limit}.}
            \item Did you discuss any potential negative societal impacts of your work?
            \answerYes{We discussed the broader societal impact in Appendix~\ref{app:impact}.}
            \item Have you read the ethics review guidelines and ensured that your paper conforms to them?
            \answerYes{}
        \end{enumerate}

        \item If you are including theoretical results...
        \begin{enumerate}
            \item Did you state the full set of assumptions of all theoretical results?
            \answerYes{The assumptions are stated in Prop.~\ref{prop:identifibility} and Prop.~\ref{prop:part-identifibility}.}
            \item Did you include complete proofs of all theoretical results?
            \answerYes{The complete proofs are in the Appendix~\ref{app:proofs}}
        \end{enumerate}

        \item If you ran experiments...
        \begin{enumerate}
            \item Did you include the code, data, and instructions needed to reproduce the main experimental results (either in the supplemental material or as a URL)?
            \answerYes{We will include the code and instructions in the supplementary files.}
            \item Did you specify all the training details (e.g., data splits, hyperparameters, how they were chosen)?
            \answerYes{The details are in the Appendix~\ref{app:hyper}.}
            \item Did you report error bars (e.g., with respect to the random seed after running experiments multiple times)?
            \answerYes{The learning curves (Fig.~\ref{fig:reward_cheetah} \textcolor{RoyalPurple}{(a-c)}, Appendix Fig.~\ref{fig:appendix_cheetah_smooth}-\ref{fig:appendix_minitaur_smooth}) include error bars with $10$ runs.}
            \item Did you include the total amount of compute and the type of resources used (e.g., type of GPUs, internal cluster, or cloud provider)?
            \answerYes{The information of computational resource is given in Appendix~\ref{app:platforms}. }
        \end{enumerate}

        \item If you are using existing assets (e.g., code, data, models) or curating/releasing new assets...
        \begin{enumerate}
            \item If your work uses existing assets, did you cite the creators?
            \answerYes{Yes, we have cited the OpenAI Gym and Mujoco libraries in our paper.}
            \item Did you mention the license of the assets?
            \answerYes{Yes, the licenses are given in Appendix~\ref{app:license}.}
            \item Did you include any new assets either in the supplemental material or as a URL?
            \answerNA{}
            \item Did you discuss whether and how consent was obtained from people whose data you're using/curating?
            \answerNA{}
            \item Did you discuss whether the data you are using/curating contains personally identifiable information or offensive content?
            \answerNA{}
        \end{enumerate}

        \item If you used crowdsourcing or conducted research with human subjects...
        \begin{enumerate}
            \item Did you include the full text of instructions given to participants and screenshots, if applicable?
            \answerNA{}
            \item Did you describe any potential participant risks, with links to Institutional Review Board (IRB) approvals, if applicable?
            \answerNA{}
            \item Did you include the estimated hourly wage paid to participants and the total amount spent on participant compensation?
            \answerNA{}
        \end{enumerate}

    \end{enumerate}


    \newpage
    \appendix
    \hrule height 3pt
    \vskip 5mm
    \begin{center}
        \Large{\textbf{Appendix for Factored Adaptation for\\ Non-Stationary Reinforcement Learning}}\\
    \end{center}
    \vskip 3mm
    \hrule height 1pt
    \setcounter{figure}{0}
    \renewcommand{\thefigure}{A\arabic{figure}}
    \setcounter{table}{0}
    \renewcommand{\thetable}{A\arabic{table}}
    \setcounter{algocf}{0}
    \renewcommand{\thealgocf}{A\arabic{algocf}}
    \renewcommand{\thesection}{\Alph{section}}
    \renewcommand{\theequation}{A\arabic{equation}}
    \setcounter{equation}{0}

    \section{Broader Impact} \label{app:impact}

    Our work is a first exploration in leveraging factored representations for non-stationary RL in order to improve adaptation. One of the limitations of our current approach, and similarly to a few other non-stationary RL approaches, is that we do not provide theoretical guarantees in terms of adapting to non-stationarity. This limits the applicability of these approaches in safety-critical applications, e.g. self-driving cars, or in adversarial environment.
    One of the future directions of this work would be to provide theoretical guarantees under reasonable assumptions, similarly to generalization bounds for factored representations for fast domain adaptation \citep{huang2021adarl}.

    Based on the application, there might be different assumptions or inductive biases that might be considered reasonable.
    Since our work leverages insights from recent causality literature \citep{CDNOD_Huang20}, we also inherit the same inductive bias in terms of assuming there is an underlying causal structure that is time-invariant throughout the non-stationarity. In our current method, this causal structure is estimated in the model estimation phase as one of the first steps of the algorithm. After this estimation, there can still be changes in the functional dependencies between the various components, but we assume there are no new edges/causal relations between components that were previously disconnected, or new forms of non-stationarity, in terms of connections between the change factor components and the state dimensions or reward. For example, if we consider Halfcheetah-v3 with a change in gravity and estimate a model that can handle this type of non-stationarity, our method will not be able to perform well under a new type of non-stationarity (e.g. change of wind forces) that was not observed during model estimation.
    An interesting extension of our work would be designing ways to efficiently detect changes in the causal structure and adapt the model.

    As is the case with other works in causal discovery, we also make some standard assumptions to recover the causal graph from time series observational data. In particular, we assume that there are no other unobserved confounders, except for the change factors, and that there are no instantaneous causal effects between the state components, which is implied by our definition of an FN-MDP and its Dynamic Bayesian Network. In practical applications, this means that we are able to measure all the relevant causal variables and we are measuring them at a rate that is faster than their interaction.
    Additionally, in our identifiability proofs, we assume the causal Markov and faithfulness assumptions \cite{SGS93}, which provide a correspondence between conditional independences and d-separations in the graph. The faithfulness assumption can be violated for example in case of deterministic relations, thus requiring a careful modelling of the system.
    In general, if these assumptions are violated, the causal structure we learn might be incorrect, and therefore the factored representation might not be beneficial to adapt to non-stationarity.
    A future direction would be to relax some of the current assumptions to provide a more realistic and flexible framework for factored non-stationary RL.

    \section{Proofs and Causality Background} \label{app:proofs}

    \subsection{Preliminaries}

    \subsubsection{Dynamic Bayesian networks}
    Dynamic Bayesian networks (DBNs)~\citep{murphy2002dynamic} are the extensions of Bayesian networks (BN), which model the time-dependent relationship between nodes (See an example in Fig.~\ref{fig:dbn}(b)). The unfolded DBNs can be represented as BNs.
    The variables in DBNs are in discrete time slices and dependent on variables from the same and previous time slices. Hence, the DBNs can model the stationary process repeated over the discrete time slices.

    \begin{figure}[h]
        \centering
        \includegraphics[width=\linewidth]{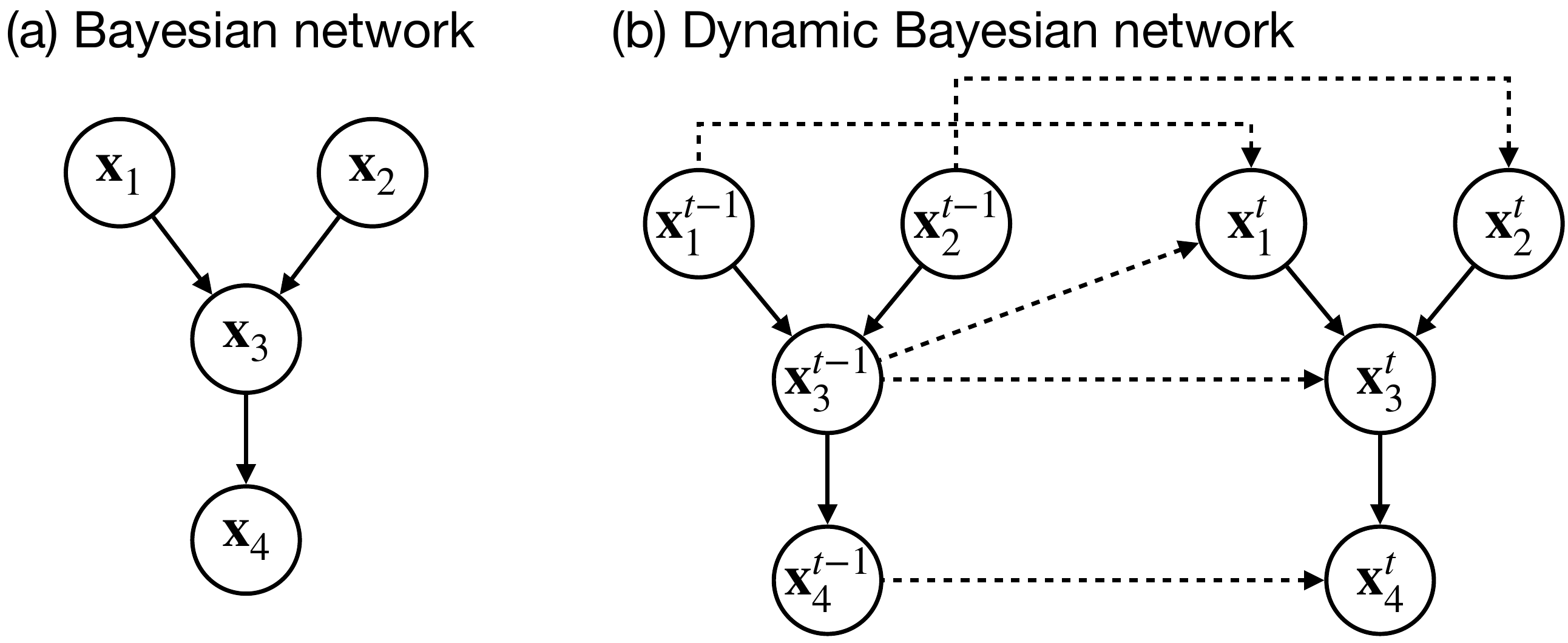}
        \caption{Examples on Bayesian and Dynamic Bayesian networks. The dashed edges indicate dependencies across time slices. }
        \label{fig:dbn}
    \end{figure}

    \subsubsection{Markov and faithfulness assumptions}
    Given a directed acyclic graph $G = (\textbf{V}, \textbf{E})$, where $\textbf{V}$ is the set of nodes and $\textbf{E}$ is the set of directed edges, we can define a graphical criterion that expresses a set of conditions on the paths.
    \begin{Definition}
        [d-separation \citep{Pearl00}]
        A path $p$ is said to be blocked by a set of nodes $\textbf{Z} \subseteq \textbf{V}$ if and only if (1) $p$ contains a chain $i \rightarrow m \rightarrow j$ or a fork $i \leftarrow m \rightarrow j$ such that the middle node $m$ is in $Z$, or (2) $p$ contains a collider $i \rightarrow m \leftarrow j$ such that the middle node $m$ is not in $\textbf{Z}$ and such that no descendant of $m$ is in $\textbf{Z}$.
        Let $\textbf{X}$, $\textbf{Y}$, and $\textbf{Z}$ be disjunct sets of nodes. $\textbf{Z}$ is said to d-separate $\textbf{X}$ from $\textbf{Y}$ (denoted as $\textbf{X} \perp_d \textbf{Y} | \textbf{Z}$) if and only if $\textbf{Z}$ blocks every path from a node in $\textbf{X}$ to a node in $\textbf{Y}$.
    \end{Definition}

    \begin{Definition}
        [Global Markov Condition \citep{SGS93, Pearl00}]
        A distribution $P$ over $\mathbf{V}$ satisfies the global Markov condition on graph $G$ if for any partition $(\textbf{X, Z, Y})$ such that $\textbf{X}$ is d-separated from $\textbf{Y}$ given $\textbf{Z}$, i.e. $\textbf{X} \perp_d \textbf{Y} | \textbf{Z}$ the distribution factorizes as:
        \begin{equation*}
            P(\textbf{X, Y}| \textbf{Z}) = P(\textbf{X} | \textbf{Z}) P(\textbf{Y} | \textbf{Z}).
        \end{equation*}
        In other words, $\textbf{X}$ is conditionally independent of $\textbf{Y}$ given $\textbf{Z}$, which we denote as $\textbf{X} \independent \textbf{Y} | \textbf{Z}$.
    \end{Definition}

    \begin{Definition}
        [Faithfulness Assumption \citep{SGS93, Pearl00}]
        There are no independencies between variables that are not entailed by the Markov Condition.
    \end{Definition}

    If we assume both of these assumptions, then we can use d-separation as a criterion to read all of the conditional independences from a given DAG $G$. In particular, for any disjoint subset of nodes $\textbf{X,Y,Z} \subseteq \textbf{V}$: $\textbf{X} \independent \textbf{Y} | \textbf{Z} \iff \textbf{X} \perp_d \textbf{Y} | \textbf{Z}$.

    \subsection{AdaRL summary} \label{sec:adarl}
    Our work extends the factored representation for fast policy adaptation across domains introduced in AdaRL~\citep{huang2021adarl}, which we summarize here. While \citet{huang2021adarl} propose a general framework that can be applied to both MDPs and POMDPs, in this work we focus on MDPs, so we only present the simplified version of AdaRL for MDPs.
    The simplified AdaRL setting considers $n$ source domains and $n'$ target domains. The state at time $t$ is represented as $\mathbf{s}_t = (s_{1,t},\cdots,s_{d,t})^\top \in \mathcal{S}^d$, while $\mathbf{a}_t \in \mathcal{A}^m$ is the executed action and $r_t \in \mathcal{R}$ is the reward signal.
    The generative process of the environment in the $k$-th domain with $k=1, \mydots, n+n'$ can be described in terms of the transition function for each dimension $i=1,\mydots, d$ of $\mathbf{s}_t$ as:
    
    \begin{equation}
        \label{eq: adarl1_new}
        s_{i, t}=f_{i}(\boldsymbol{c}_{i}^{\boldsymbol{s} \scriptveryshortarrow \boldsymbol{s}} \odot \boldsymbol{s}_{t-1},
        \boldsymbol{c}_{i}^{\boldsymbol{a} \scriptveryshortarrow \boldsymbol{s}} \odot \boldsymbol{a}_{t-1},
        \boldsymbol{c}_{i}^{\boldsymbol{\theta}_{k} \scriptveryshortarrow \boldsymbol{s}} \odot \boldsymbol{\theta}_{k}^{\boldsymbol{s}}, \epsilon_{i, t}^{s})
    \end{equation}
    
    where $\odot$ denotes the element-wise product. The binary mask $\boldsymbol{c}_{i}^{\boldsymbol{s} \scriptveryshortarrow \boldsymbol{s}} \in \{0, 1\}^d$ represents which of the state components $s_{j,t-1}$ are used in the transition function of $s_{i,t}$.
    Similarly, $\boldsymbol{c}_{i}^{\boldsymbol{a} \scriptveryshortarrow \boldsymbol{s}} \in \{0, 1\}^m$ is a mask that indicates whether the action directly affects $s_{i,t}$.
    The change factor $\boldsymbol{\theta}_{k}^{\boldsymbol{s}} \in \mathbb{R}^p$ is the only parameter that depends on the domain $k$ in Eq.~\ref{eq: adarl1} and it encodes any change across domains in the dynamics. The binary mask $\boldsymbol{c_{i}}^{\boldsymbol{\theta}_{k} \scriptveryshortarrow \boldsymbol{s}} \in \{0,1\}^p$ represents which of the $\boldsymbol{\theta}_{k}^{\boldsymbol{s}}$ components influence the $s_{i,t}$. Finally, $\epsilon^s_{i,t}$ is an i.i.d. random noise.
    Similarly the reward function is modeled as:
    
    \begin{equation}
        r_{t}=h (\boldsymbol{c}^{\boldsymbol{s} \scriptveryshortarrow r} \odot \boldsymbol{s}_{t-1}, \boldsymbol{c}^{\boldsymbol{a} \scriptveryshortarrow r} \odot \boldsymbol{a}_{t-1}, \boldsymbol{\theta}_{k}^{r}, \epsilon_{t}^{r} )
        \label{eq: adarl2}
    \end{equation}
    
    where $\boldsymbol{c}_{i}^{\boldsymbol{s} \scriptveryshortarrow \boldsymbol{r}} \in \{0, 1\}^d$, $\boldsymbol{c}_{i}^{\boldsymbol{a} \scriptveryshortarrow \boldsymbol{s}} \in \{0, 1\}^m$, and $\epsilon_{t}^{r}$ is an i.i.d. random noise. The change factor $\boldsymbol{\theta}_{k}^{r} \in \mathbb{R}^q$ is the only parameter that depends on the domain $k$ in Eq.~\ref{eq: adarl2} and it encodes any change in the reward function.
    In this simplified setting, the binary masks $\boldsymbol{c}^{\cdot \veryshortarrow \cdot}$ can be seen as indicators of edges in a Dynamic Bayesian Network (DBN). Under Markov and faithfulness assumptions, i.e., assuming the conditional independences in the data and d-separations in the true underlying graph coincide, the edges in the graph can be uniquely identified. This means one can learn the true causal graph representing jointly all of the environments, even if the change parameters are latent.

    In the general AdaRL framework, the representation is learned via a combination of a state prediction network (to estimate the various $f_i$) and a reward prediction network (to estimate $h$). All binary masks $\boldsymbol{c}^{\cdot \scriptveryshortarrow \cdot}$ and change factors $\theta_k^\cdot$ are trainable parameters. All change factors $\theta_k^\cdot$ are assumed to be constant in each domain $k$.
    If the inputs are pixels, another encoder is added to infer the symbolic states, forming a \emph{Multi-model Structured Sequential Variational Auto-Encoder (MiSS-VAE)}. MiSS-VAE leverages the generative modeling to learn the data generation process in RL system with multiple domains. All change factors in MiSS-VAE are modeled by constants in each domain. This setting is not suitable for non-stationary RL where the change factors evolve over time.

    In general, not all of the dimensions of the learned state and change factor vectors are useful in policy learning.
    \citet{huang2021adarl} select a subsection of dimensions which are essential for policy optimization. Leveraging the learned representation as a DBN, we can select compact states and CFs as having a directed path to a reward:
    
    \begin{align*}
        s_{i,t} \in \boldsymbol{s}^{min} &\iff s_{i,t} \to \mydots \to r_{t+\tau} \text{ for } \tau \geq 1 \\
        \theta_{i} \in \boldsymbol{\theta}^{min} &\iff \theta_{i}  \to \mydots \to r_{t+\tau} \text{ for } \tau \geq 1
    \end{align*}

    \subsection{Proofs}

    \begin{proposition}[Identifiability with observed change factors]
        Suppose all the change factors $\boldsymbol{\theta^s}_{t}$ and $\boldsymbol{\theta^r}_{t}$ are observed, i.e., Eq.~\ref{eq: adarl1}-\ref{Eq: conti_model}) is an MDP. Under the Markov and faithfulness assumptions, all the binary masks $\boldsymbol{C}^{\cdot \veryshortarrow \cdot}$ are identifiable.
    \end{proposition}
    \begin{proof}
        We construct the graph in the Factored Non-stationary MDP as a dynamic Bayesian network (DBN) $G$ over the variables $\mathbf{V_{MDP}} =\{s_{1,t-1}, \dots, s_{d,t-1}, s_{1,t}, \dots, s_{d,t}, a_{1,t-1}, \dots, a_{m, t-1}, r_{t-1} \}$,
        and the change factors $\mathbf{V_{\theta}} =  \{\theta^s_{1,t-1}, \dots, \theta^s_{p,t-1}, \theta^r_{1,t-1}, \dots, \theta^r_{q,t-1}, \theta^s_{1,t}, \dots, \theta^s_{p,t}, \theta^r_{1,t}, \dots, \theta^r_{q,t}\}$.
        In this setting, we can always the correct causal graph under the Markov and faithfulness assumptions.
        We rewrite the time index of $r_{t-1}$ as $r_{t}$.

        We could rewrite also the change factors time indices by shifting them back in time by one step, i.e. $t \to t-1$, and this would allow us to model the whole setup as a Markov DBN without any instantaneous effect. In this setup, we can leverage existing results to show that the true causal graph is asymptotically identifiable from conditional independences in a time-series without any unobserved confounders or instantaneous effects \citep{peters2013causal}. In order to make our proof clearer, we instead show step by step how we can recover the parts of the graph related to $\mathbf{V_{MDP}}$, to  $\mathbf{V_{\theta}}$ and finally the connections between them.

        In this setting, there are no instantaneous effects except for the change factors, and the only causal parents for a variable that is not a change factor at time $t$ can be in the previous time-step $t-1$.
        In particular, as in usual MDPs, these are the only allowed edges:
        \begin{compactenum}
            \item state dimension $s_{i,t-1}$ at time $t-1$ to state dimension $s_{j,t}$ at time $t$, for $i,j \in \{1, \dots, d\}$ (this includes the case in which $i = j$);
            \item action dimension $a_{k, t-1}$ at time $t-1$ to state dimension $a_{j, t}$ at time $t$, for $j \in \{1, \dots, d\}, k \in \{1, \dots, m\}$;
            \item state dimension $s_{i,t-1}$ at time $t-1$ to reward $r_{t}$ at time $t$, for $i \in \{1, \dots, d\}, k \in \{1, \dots, m\}$;
            \item action dimension $a_{k, t-1}$ at time $t-1$ to reward $r_{t}$ at time $t$, for $k \in \{1, \dots, m\}$;
        \end{compactenum}
        and in addition we have some extra knowledge about the allowed edges to and from change factors, as expressed in our generative model in Equations (1-3):
        \begin{compactenum}
            \item transition change factor dimension $\theta^s_{i,t}$ at time $t$ to state dimension $s_{j,t}$ at time $t$, for $i \in \{1, \dots, p\}, j \in \{1, \dots, d\}$;
            \item transition change factor dimension $\theta^s_{i,t}$ at time $t$ to transition change factor dimension $\theta^s_{j,t+1}$ at time $t+1$, for $i,j \in \{1, \dots, p\}$;
            \item reward change factor dimension $\theta^r_{i,t}$ at time $t$ to reward change factor dimension $\theta^r_{j,t+1}$ at time $t+1$, for $i,j \in \{1, \dots, q\}$;
        \end{compactenum}
        For the reward change factors, we assume they are fully connected to the reward, as shown in Eq. (2).
        Using this background knowledge, we can learn the edges from any other variable $V_{i,t-1} \to V_{j,t}$, for $V_{i,t-1}, V_{j,t} \in \mathbf{V_{MDP}}$, just by checking if $V_{i,t-1} \dependent V_{i,t} | \mathbf{\theta}^r_{t}, \mathbf{\theta}^s_{t}, \mathbf{s}_{t-2}, \mathbf{a}_{t-2}$. This dependence implies that the variables are d-connected, under the Markov and faithfulness assumptions. Except for a direct edge, there is no other possible path through the graph unrolled in time, since we have blocked all influence of time-step $t-2$ and earlier, and the paths through future time-steps contain colliders. So this means that $V_{i,t-1}$ and $ V_{j,t}$ are adjacent, and in particular the edge follows the arrow of time, from $t-1$ to $t$. In our setting, we do not model the actions at time $t$, since we assume they are not caused by any other variable. We also never need to condition on the reward to check if two variables are adjacent, since it's always on a collider path.
        This means we are able to learn the following binary masks:
        \begin{compactenum}
            \item state dimensions to state dimensions $\boldsymbol{C}^{\boldsymbol{s} \scriptveryshortarrow \boldsymbol{s}}$;
            \item action dimensions to state dimensions $\boldsymbol{C}^{\boldsymbol{a} \scriptveryshortarrow \boldsymbol{s}}$;
            \item state dimensions to reward $\boldsymbol{c}^{\boldsymbol{s} \scriptveryshortarrow r}$
            \item action dimensions to reward $\boldsymbol{c}^{\boldsymbol{a} \veryshortarrow r}$;
        \end{compactenum}

        To learn the edges between any change factor component to another change factor component, i.e., $V_{i,t-1} \to V_{j,t}$ for $ V_{i,t-1}, V_{j,t}\in \mathbf{V_{\theta}}$, we can just check if $V_{i,t-1} \dependent V_{i,t} | \mathbf{\theta}^r_{t-2}, \mathbf{\theta}^s_{t-2}$, since states, actions and rewards can never be parents of the change factors, so we do not need to condition on them to close any path through the earlier time-steps. We also assumed that change factor components follow a Markov process, so they do not have instantaneous effects towards each other.
        This means we are able to learn the following masks:
        \begin{compactenum}
            \item transition change factor dimensions to transition change factor dimensions $\boldsymbol{C}^{\boldsymbol{\theta^s} \scriptveryshortarrow \boldsymbol{\theta^s}}$;
            \item reward change factor dimensions to reward change factor dimensions $\boldsymbol{C}^{\boldsymbol{\theta^r} \scriptveryshortarrow \boldsymbol{\theta^r}}$ ;
        \end{compactenum}

        Finally to learn the edges between the change factors $V_{i,t} \in \mathbf{V_{\theta}}$ and the other variables $V_{j,t} \in \mathbf{V_{MDP}}$, we can just check if  $V_{i,t} \dependent V_{j,t} |\mathbf{s}_{t-1}, \mathbf{a}_{t-1}, \mathbf{\theta}^r_{t-1}, \mathbf{\theta}^s_{t-1}$ and if this is true we can learn the edge $V_{i,t} \to V_{j,t}$.
        This means we are able to learn the mask:
        \begin{compactenum}
            \item transition change factor dimensions to state dimensions $\boldsymbol{C}^{\boldsymbol{\theta^s} \scriptveryshortarrow \boldsymbol{s}}$;
        \end{compactenum}
        All of these results together show that if we have the Markov and faithfulness assumption, i.e. the conditional independence tests return the true d-separations in the graph, and we observe the change factors, we are able to completely identify the graph of the FN-MDP represented by the binary masks.
    \end{proof}

    \begin{proposition}[Partial Identifiability with latent change factors]
        Suppose the generative process follows Eq.~\ref{eq: adarl1}-\ref{Eq: conti_model} and the change factors $\boldsymbol{\theta^s}_{t}$ and $\boldsymbol{\theta^r}_{t}$ are unobserved.
        Under the Markov and faithfulness assumptions, the binary masks $\boldsymbol{C}^{\boldsymbol{s} \scriptveryshortarrow \boldsymbol{s}}, \boldsymbol{C}^{\boldsymbol{a} \scriptveryshortarrow \boldsymbol{s}}, \boldsymbol{c}^{\boldsymbol{s} \scriptveryshortarrow r}$ and $\boldsymbol{c}^{\boldsymbol{a} \veryshortarrow r}$. Moreover, we can identify which state dimensions are affected by $\boldsymbol{\theta^s}_{t}$ and if the reward function changes.
    \end{proposition}

    \begin{proof}
        In this case, the MDP is non-stationary, since we cannot observe the latent change factors. We assume that we can represent the latent change factors as a smooth function of the observed time index $t$. This assumption is called \emph{pseudo-causal sufficiency} in previous work \citep{CDNOD_Huang20}.
        We can then use the time index $t$ as a surrogate variable to characterize the unobserved change factors, since at each time-step their value will be fixed.

        We again consider a DBN $\mathcal{G}_{MDP  \cup \{t\}}$ over the variables $\mathbf{V_{MDP}} =\{s_{1,t-1}, \dots, s_{d,t-1}, s_{1,t}, \dots, s_{d,t}, a_{1,t-1}, \dots, a_{m, t-1}, r_{t-1} \}$ and the time index $t$.
        We rewrite the time index of $r_{t-1}$ as $r_{t}$. Note that we do not represent the change factors in this DBN, but we can capture their effect through $t$ since they are assumed to be deterministic smooth functions of the time index.

        We can then reuse the results by \citet{CDNOD_Huang20} (Theorem 1) in which under the pseudo-causal sufficiency (Assumption 1 in that paper) and the Markov and faithfulness assumption (Assumption 2), one can asymptotically identify the true causal skeleton (i.e. the adjacencies) in the graph $\mathcal{G}_{MDP \cup \{t\} }$) through conditional independence tests.
        In particular for any $V_i, V_j \in \mathbf{V}_{MDP} \cup \{t \}$,  $V_i$ and $V_j$ are not adjacent if there exists a subset of the unrolled graph
        $\mathbf{V}_k$ of $\mathbf{V}_{MDP \cup \{t \}} \setminus \{V_i, V_j\}$ in $\mathcal{G}_{MDP \cup \{t\} }$ such that  $V_i \independent V_j |\mathbf{V}_k$.
        In particular, we can also focus on the tests that were used for the proof of the previous Proposition, by just substituting the change factors with the time index $t$.
        The skeleton is the undirected version of the causal graph, so this result only tells us that can asymptotically get the correct undirected edges, but not their orientations.

        Fortunately, in our setting we have some additional background knowledge that allows us to orient all the existing edges.
        In particular, as in usual MDPs, these are the only allowed edges:
        \begin{compactenum}
            \item state dimension $s_{i,t-1}$ at time $t-1$ to state dimension $s_{j,t}$ at time $t$, for $i,j \in \{1, \dots, d\}$ (this includes the case in which $i = j$);
            \item action dimension $a_{k, t-1}$ at time $t-1$ to state dimension $a_{j, t}$ at time $t$, for $j \in \{1, \dots, d\}, k \in \{1, \dots, m\}$;
            \item state dimension $s_{i,t-1}$ at time $t-1$ to reward $r_{t}$ at time $t$, for $i \in \{1, \dots, d\}, k \in \{1, \dots, m\}$;
            \item action dimension $a_{k, t-1}$ at time $t-1$ to reward $r_{t}$ at time $t$, for $k \in \{1, \dots, m\}$;
        \end{compactenum}
        This means that, for example, we cannot have a variable at time $t$ causing a variable at time $t-1$. Therefore if two variables $V_{i, t-1}$ and $V_{j, t}$ are adjacent, we already know that the direction of that edge will be $V_{i, t-1} \to V_{j, t}$.
        This also implies that the following binary masks are identifiable (i.e. no edge remains unoriented):
        \begin{compactenum}
            \item state dimensions to state dimensions $\boldsymbol{C}^{\boldsymbol{s} \scriptveryshortarrow \boldsymbol{s}}$;
            \item action dimensions to state dimensions $\boldsymbol{C}^{\boldsymbol{a} \scriptveryshortarrow \boldsymbol{s}}$;
            \item state dimensions to reward $\boldsymbol{c}^{\boldsymbol{s} \scriptveryshortarrow r}$
            \item action dimensions to reward $\boldsymbol{c}^{\boldsymbol{a} \veryshortarrow r}$;
        \end{compactenum}
        These represent all of the edges in $\mathcal{G}_{MDP}$. We can also learn the edges from $t$ to $V_{MDP}$ (by construction we assume the opposite direction is not possible), which will represent the effect of the change factors, as we show in the following.

        Since $t$ inherits all of the children of the latent change factors in $G$, we can further show that if $s_{i,t} \independent t| \boldsymbol{s}_{t-1}, \boldsymbol{a}_{t-1}$ in $G$, then none of the latent change factor dimensions $\theta^s_{j, t}$ affect $s_{i,t}$, i.e., $s_{i,t} \independent t| \boldsymbol{s}_{t-1}, \boldsymbol{a}_{t-1} \iff c_{i,j}^{\boldsymbol{\theta^s} \scriptveryshortarrow \boldsymbol{s}} = 0$.
        Intuitively, this means that the distribution of $s_{i,t}$ only depends on $\mathbf{s}_{t-1}$ and $\mathbf{a}_{t-1}$, and not on the timestep $t$, or in other words, this distribution is stationary.
        Under the same principle, if $r_{t} \independent t| \boldsymbol{s}_{t}, \boldsymbol{a}_t$, then the reward is stationary.

    \end{proof}

    \section{Details on Experimental Designs and Results}
    \label{appendix_design}

    \subsection{MuJoCo} \label{app:mujoco}
    We modify the Half-Cheetah environment into a variety of non-stationary settings. Details on the change factors are given as below.

    \paragraph{Changes on dynamics.}
    We change the wind forces $f^w$ in the environment. We consider the changing functions can be both continuous and discrete.

    $\bullet$ Continuous changes: $f^w_t = 10 + 10 \sin(0.005 \cdot t)$, where $t$ is the timestep index; \\
    $\bullet$ Discrete changes:
    (1) Across-episode: \\
    a. Sine function: $f_w = 10 + 10 \sin(0.5 \cdot i)$ \\
    b. Damping-like function: $f_{w}=10+3 \cdot (1.01)^{-\lceil i/10 \rceil} \sin (0.5 \cdot i)$ \\
    c. Piecewise linear function $f_{w}=5+0.02\cdot\|i-1500\|$; , where $i$ is the episode index.\\
    (2) Within-episode: $f_w = 10 + 10 \sin(0.4 \cdot \lfloor t / 10\rfloor)$, where $t$ is the timestep index.

    We also consider a special case where the agent's mechanism is changing over time. Specifically, the one random joint is disabled at the beginning of each episode. 

    \paragraph{Changes on reward functions.} To introduce non-stationarity in the rewards, we change the target speed $v_g$ in each episode. To make the learning process stable, we only consider the discrete changes and the change points are located at the beginning of each episode. The changing function is $v_{g}=1.5+1.5 \sin (0.2 \cdot i)$, where $i$ denotes the episode index.

    \paragraph{Changes on both dynamics and rewards.} We consider a more general but challenging scenario, where the changes on dynamics and rewards can happen concurrently during the lifetime of the agents. We change the wind forces and target speed at the beginning of each episode. At episode $i$, the dynamics and reward functions are:
    
    \begin{equation*}
        \left\{\begin{array}{l}
                   f_w = 10 + 10 \sin(w \cdot i) \\
                   v_{g} = 1.5+1.5 \sin (w \cdot i)
        \end{array}\right.
    \end{equation*}
    
    Here, $w$ is the non-stationary degree.
    We consider multiple values of $w$ in our experiments. In Fig.~\ref{fig:reward_cheetah}\textcolor{RoyalPurple}{(d)}, $w=0.5$. In Fig.~\ref{fig:reward_cheetah}\textcolor{RoyalPurple}{(h)}, $w$ is the value of non-stationary degree.

    \subsection{Sawyer benchmarks} \label{app:sawyer}
    In Sawyer-Reaching, the sawyer arm is trained to reach a target position $\boldsymbol{s}^{g}_{t}$.
    The reward $r_t$ is the difference between the current position $\boldsymbol{s}_t$ and the target position $r_t=-\left\|\boldsymbol{s}_t-\boldsymbol{s}^{g}\right\|_{2}$. In this task, we cannot directly modify the dynamics in the simulator, so consider a reward-varying scenario, where the target location changes across each episode following a periodic function.
    In Sawyer-Peg, the robot arm is trained to insert a peg into a designed target location $\boldsymbol{s}^g$. The reward function is $r_t=\mathbb{I}\left(\|\boldsymbol{s_t}-\boldsymbol{s}^g\|_{2} \leq 0.05\right)$.

    We change the target location in Sawyer reaching task. The target location $\boldsymbol{s}^g_t$ is given as below:
    
    \begin{equation*}
        \mathbf{s}^{g}_t=\left[\begin{array}{c}
                                   0.1 \cdot \| \cos (0.2 \cdot i) \| \\
                                   0.1 \cdot \sin (0.5 \cdot i)       \\
                                   0.2
        \end{array}\right]
    \end{equation*}
    
    where $i$ is the episode index.
    For Sawyer-Peg task, the target location $\boldsymbol{s}_g$ changes at each episode. The parameters in each dimension of $\boldsymbol{s}_g$ is randomly sampled at episode $i$ as below:

    $\bullet$ x\_range\_1: $(0.44,0.45)$;\\
    $\bullet$ x\_range\_2: $(0.6,0.61)$;\\
    $\bullet$ y\_range\_1: $(-0.08,-0.07)$;\\
    $\bullet$ y\_range\_2: $(0.07,0.08)$;

    \subsection{Minitaur benchmarks} \label{app:minitaur}
    We consider both the changes on dynamics and reward functions.

    \paragraph{Changes on dynamics.} We change the mass of taur $m$ in the environment. Specifically, we consider both the continuous and discrete changes.

    $\bullet$ Continuous changes: $m_t=1.0+ 0.75 \sin (0.005 \cdot t)$; \\
    $\bullet$ Discrete and within-episode changes: $m_t=1.0+ 0.75 \sin(0.3 \cdot \lfloor t / 20\rfloor)$

    \paragraph{Changes on both dynamics and reward functions.} We also consider a case where both the dynamics and reward functions change at the beginning of each episode. We change the target speed of minitaur to introduce the non-stationarity of reward functions. The change functions are given below:
    
    \begin{equation*}
        \left\{\begin{array}{l}
                   m_i = 1.0 + 0.5 \sin(0.5 \cdot i) \\
                   s_{v} = 0.3+0.2 \sin (0.5 \cdot i)
        \end{array}\right.
    \end{equation*}

    \subsection{Full results} \label{app:full}
    Fig.~\ref{fig:appendix_cheetah_smooth} and~\ref{fig:appendix_minitaur_smooth} give the smoothed learning curves on average return over 10 runs versus timesteps in Half-Cheetah and Minitaur experiments.   Table~\ref{tbl:full} shows the average final return over 10 runs for all experiments.  Fig.~\ref{fig:non_s_appendix} demonstrates the return on Half-Cheetah with different non-stationary degrees on multi-factor changing scenario.
    Fig.~\ref{fig:appendix_latent} gives average return on different benchmarks with varying numbers of latent features with all evaluated approaches. 

    \begin{figure}[H]
        \centering
        \includegraphics[width=0.95\linewidth]{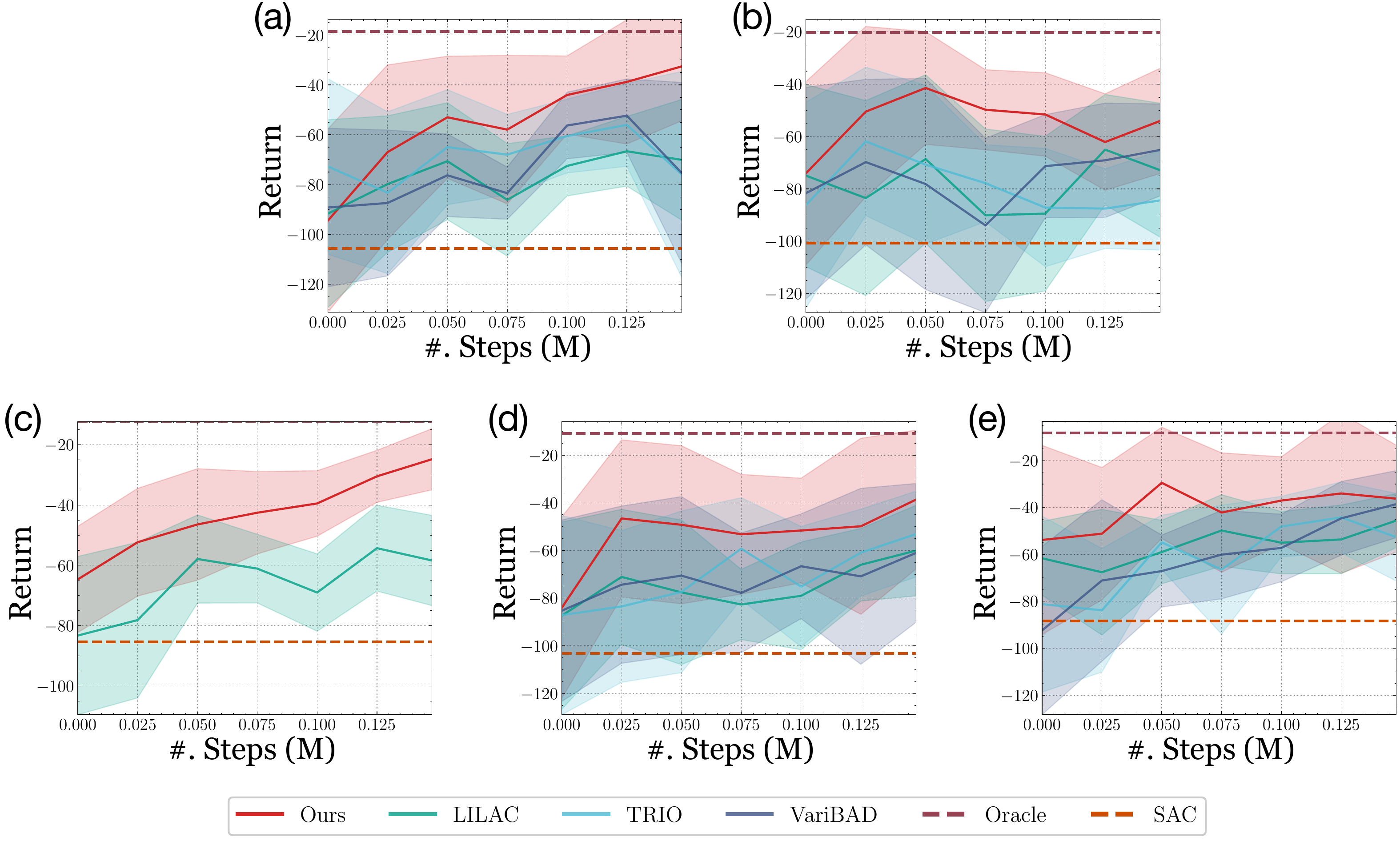}
        \caption{The average return (smoothed) across timesteps in Half-Cheetah experiments. (a) Discrete (across-episode) changes on wind forces; (b) Discrete (within-episode) changes on wind forces; (c) Continuous changes on wind forces; (d) Discrete (across-episode) changes on target speed; (e) Discrete (across-episode) changes on wind forces and target speed concurrently. }
        \label{fig:appendix_cheetah_smooth}
    \end{figure}

    \begin{figure}[H]
        \centering
        \includegraphics[width=0.95\linewidth]{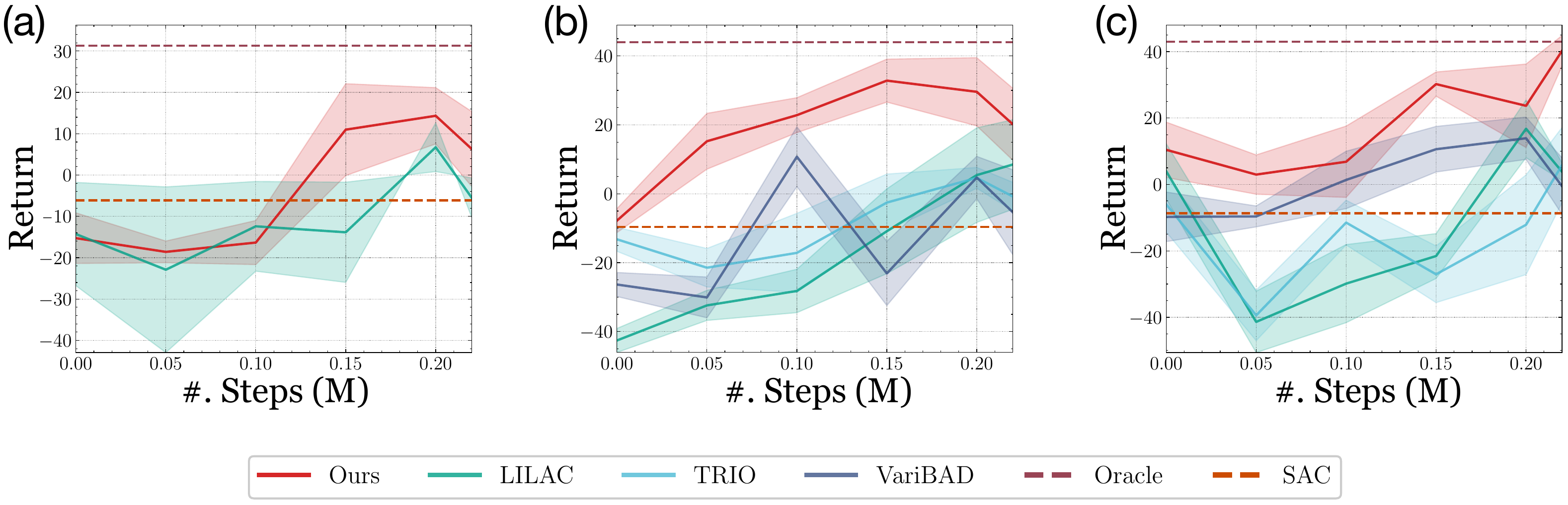}
        \caption{The average return (smoothed) across timesteps in Minitaur experiments. (a) Continuous changes on the mass; (b) Discrete (across-episode) changes on the target speed; (c) Discrete (across-episode) changes on mass and target speed concurrently.}
        \label{fig:appendix_minitaur_smooth}
    \end{figure}

    \begin{table}[H]
        \centering
        \begin{tabular}{c|cccccc}
            \toprule
            & Oracle & SAC & LILAC & TRIO & VariBAD & Ours \\ \midrule
            Half-Cheetah: A-EP (D) & \begin{tabular}[c]{@{}c@{}}
                                            $-24.4$\\ ($\pm 16.2$)
            \end{tabular} & \begin{tabular}[c]{@{}c@{}}
                                $-113.4$ $\bullet$\\ ($\pm 28.5$)
            \end{tabular} & \begin{tabular}[c]{@{}c@{}}
                                $-70.1$ $\bullet$\\ ($\pm 27.7$)
            \end{tabular} & \begin{tabular}[c]{@{}c@{}}
                                $-76.0$ $\bullet$\\ ($\pm 47.3$)
            \end{tabular} & \begin{tabular}[c]{@{}c@{}}
                                $-75.5$ $\bullet$\\ ($\pm 41.6$)
            \end{tabular} & \begin{tabular}[c]{@{}c@{}}
                                $\textbf{-32.6}$\\ ($\pm 25.0$)
            \end{tabular} \\ \hline
            Half-Cheetah: A-EP (A) & \begin{tabular}[c]{@{}c@{}}
                                            $-9.6$\\ ($\pm 5.7$)
            \end{tabular} & \begin{tabular}[c]{@{}c@{}}
                                $-30.5$ $\bullet$\\ ($\pm 12.1$)
            \end{tabular} & \begin{tabular}[c]{@{}c@{}}
                                $-19.4$ $\bullet$\\ ($\pm 11.4$)
            \end{tabular} & \begin{tabular}[c]{@{}c@{}}
                                $-21.9$ $\bullet$\\ ($\pm 13.0$)
            \end{tabular} & \begin{tabular}[c]{@{}c@{}}
                                $-17.3$ $\bullet$\\ ($\pm 10.2$)
            \end{tabular} & \begin{tabular}[c]{@{}c@{}}
                                $\textbf{-15.1}$\\ ($\pm 9.8$)
            \end{tabular} \\ \hline
            Half-Cheetah: W-EP (D) & \begin{tabular}[c]{@{}c@{}}
                                         $-48.2$\\ ($\pm 41.6$)
            \end{tabular} & \begin{tabular}[c]{@{}c@{}}
                                $-107.5$ $\bullet$\\ ($\pm 20.6$)
            \end{tabular} & \begin{tabular}[c]{@{}c@{}}
                                $-72.9$ $\bullet$\\ ($\pm 29.3$)
            \end{tabular} & \begin{tabular}[c]{@{}c@{}}
                                $-84.4$ $\bullet$\\ ($\pm 21.7$)
            \end{tabular} & \begin{tabular}[c]{@{}c@{}}
                                $-65.1$ $\bullet$\\ ($\pm 20.1$)
            \end{tabular} & \begin{tabular}[c]{@{}c@{}}
                                $\textbf{-54.0}$\\ ($\pm 23.0$)
            \end{tabular} \\ \hline
            Half-Cheetah: CONT (D) & \begin{tabular}[c]{@{}c@{}}
                                         $-12.3$\\ ($\pm 27.7$)
            \end{tabular} & \begin{tabular}[c]{@{}c@{}}
                                $-112.0$ $\bullet$\\ ($\pm 16.9$)
            \end{tabular} & \begin{tabular}[c]{@{}c@{}}
                                $-58.4$ $\bullet$\\ ($\pm 22.3$)
            \end{tabular} & - & - & \begin{tabular}[c]{@{}c@{}}
                                        $\textbf{-24.8}$\\ ($\pm 21.1$)
            \end{tabular} \\ \hline
            Half-Cheetah: A-EP (R) & \begin{tabular}[c]{@{}c@{}}
                                         $-10.9 $\\ ($\pm 20.1$)
            \end{tabular} & \begin{tabular}[c]{@{}c@{}}
                                $-131.5$ $\bullet$\\ ($\pm 16.9$)
            \end{tabular} & \begin{tabular}[c]{@{}c@{}}
                                $-60.1$ $\bullet$\\ ($\pm 21.7$)
            \end{tabular} & \begin{tabular}[c]{@{}c@{}}
                                $-53.1$ $\bullet$\\ ($\pm 20.6$)
            \end{tabular} & \begin{tabular}[c]{@{}c@{}}
                                $-61.0$ $\bullet$\\ ($\pm 33.3$)
            \end{tabular} & \begin{tabular}[c]{@{}c@{}}
                                $\textbf{-38.7}$\\ ($\pm 33.3$)
            \end{tabular} \\ \hline
            Half-Cheetah: A-EP (R+D) & \begin{tabular}[c]{@{}c@{}}
                                           $-15.2$\\ ($\pm 38.1$)
            \end{tabular} & \begin{tabular}[c]{@{}c@{}}
                                $-105.3$ $\bullet$\\ ($\pm 38.1$)
            \end{tabular} & \begin{tabular}[c]{@{}c@{}}
                                $-45.6$\\ ($\pm 13.1$)
            \end{tabular} & \begin{tabular}[c]{@{}c@{}}
                                $-52.6$ $\bullet$\\ ($\pm 21.4$)
            \end{tabular} & \begin{tabular}[c]{@{}c@{}}
                                $-38.6$\\ ($\pm 16.3$)
            \end{tabular} & \begin{tabular}[c]{@{}c@{}}
                                $\textbf{-36.2}$\\ ($\pm 26.0$)
            \end{tabular} \\ \hline
            Sawyer-Reaching: A-EP (R) & \begin{tabular}[c]{@{}c@{}}
                                            $6.4 $\\ ($\pm 3.9$)
            \end{tabular} & \begin{tabular}[c]{@{}c@{}}
                                $-52.5$ $\bullet$\\ ($\pm 9.1$)
            \end{tabular} & \begin{tabular}[c]{@{}c@{}}
                                $-34.0$ $\bullet$\\ ($\pm 8.2$)
            \end{tabular} & \begin{tabular}[c]{@{}c@{}}
                                $-28.1$ $\bullet$\\ ($\pm 2.9$)
            \end{tabular} & \begin{tabular}[c]{@{}c@{}}
                                $-31.3$ $\bullet$\\ ($\pm 4.3$)
            \end{tabular} & \begin{tabular}[c]{@{}c@{}}
                                $\textbf{-9.7}$\\ ($\pm 2.5$)
            \end{tabular} \\ \hline
            Minitaur: CONT (D) & \begin{tabular}[c]{@{}c@{}}
                                     $31.3$\\ ($\pm 4.2$)
            \end{tabular} & \begin{tabular}[c]{@{}c@{}}
                                $-6.1$ $\bullet$\\ ($\pm 3.9$)
            \end{tabular} & \begin{tabular}[c]{@{}c@{}}
                                $-5.5$ $\bullet$\\ ($\pm 11.7$)
            \end{tabular} & - & - & \begin{tabular}[c]{@{}c@{}}
                                        $\textbf{6.3}$\\ ($\pm 10.4$)
            \end{tabular} \\ \hline
            Minitaur: W-EP (D) & \begin{tabular}[c]{@{}c@{}}
                                     $44.9$\\ ($\pm 5.8$)
            \end{tabular} & \begin{tabular}[c]{@{}c@{}}
                                $-9.6$ $\bullet$ \\ ($\pm 5.5$)
            \end{tabular} & \begin{tabular}[c]{@{}c@{}}
                                $8.5$ $\bullet$\\ ($\pm 14.9$)
            \end{tabular} & \begin{tabular}[c]{@{}c@{}}
                                $-0.8$ $\bullet$\\ ($\pm 4.7$)
            \end{tabular} & \begin{tabular}[c]{@{}c@{}}
                                $5.4$ $\bullet$\\ ($\pm 14.1$)
            \end{tabular} & \begin{tabular}[c]{@{}c@{}}
                                $\textbf{20.2}$\\ ($\pm 11.9$)
            \end{tabular} \\ \hline
            Minitaur: A-EP (R+D) & \begin{tabular}[c]{@{}c@{}}
                                       $43.0$\\ ($\pm 4.7$)
            \end{tabular} & \begin{tabular}[c]{@{}c@{}}
                                $-8.7$ $\bullet$ \\ ($\pm 5.4$)
            \end{tabular} & \begin{tabular}[c]{@{}c@{}}
                                $3.8$ $\bullet$\\ ($\pm 3.0$)
            \end{tabular} & \begin{tabular}[c]{@{}c@{}}
                                $5.8$ $\bullet$\\ ($\pm 12.9$)
            \end{tabular} & \begin{tabular}[c]{@{}c@{}}
                                $21.5$ $\bullet$\\ ($\pm 9.7$)
            \end{tabular} & \begin{tabular}[c]{@{}c@{}}
                                $\textbf{40.2}$\\ ($\pm 5.3$)
            \end{tabular} \\ \bottomrule
        \end{tabular}
        \label{tbl:full}
        \caption{Average final return of different methods on Half-Cheetah, Sawyer-Reaching, and minitaur benchmarks with a variety of non-stationary settings. The best non-oracle results w.r.t. the mean are marked in  \textbf{bold}. "$\bullet$" indicates the baseline for which the improvements of our approach are statistically significant (via Wilcoxon signed-rank test at $5\%$ significance level). D, R, and A denote changes on dynamics, reward and agent's mechanism respectively. A-EP, W-EP, and CONT denote across-episode, within-episode and continuous changes, respectively.}
    \end{table}
    \begin{figure}[H]
        \centering
        \includegraphics[width=0.4\linewidth]{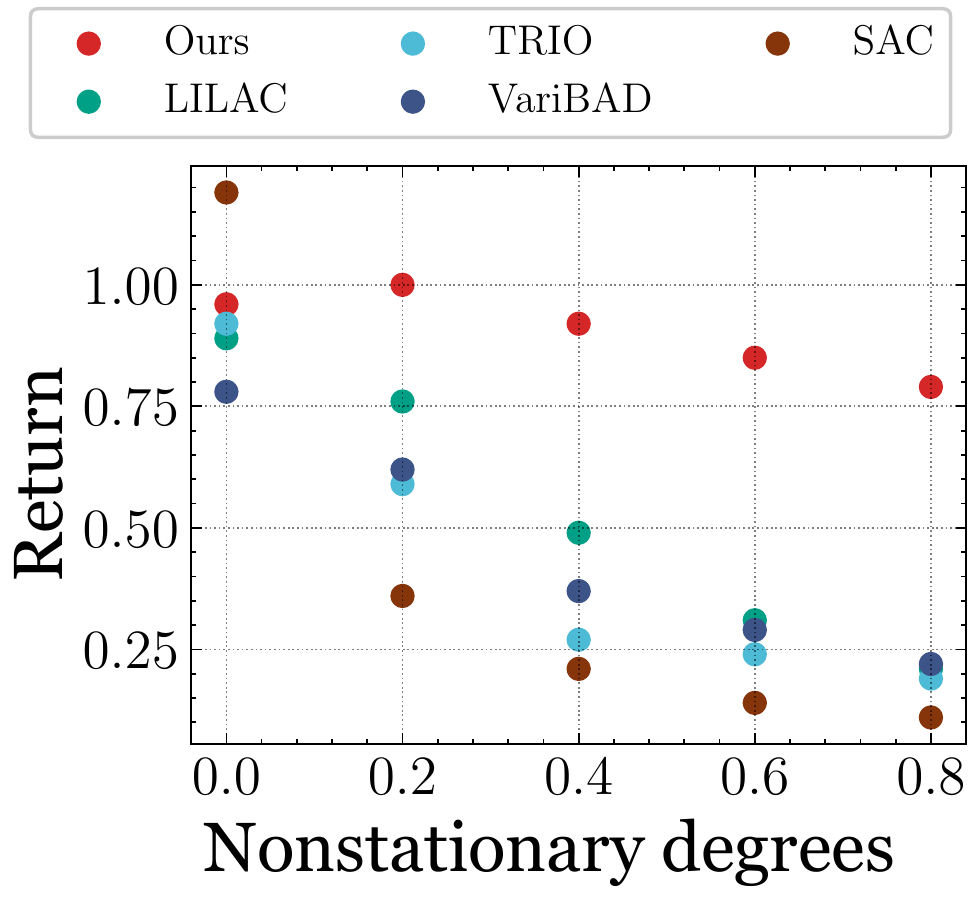}
        \caption{Average final return on 10 runs on Half-Cheetah with different non-stationary degrees on across-episode and multi-factor changes.}
        \label{fig:non_s_appendix}
    \end{figure}

    \begin{figure}[H]
        \begin{center}
            \includegraphics[width=1\linewidth]{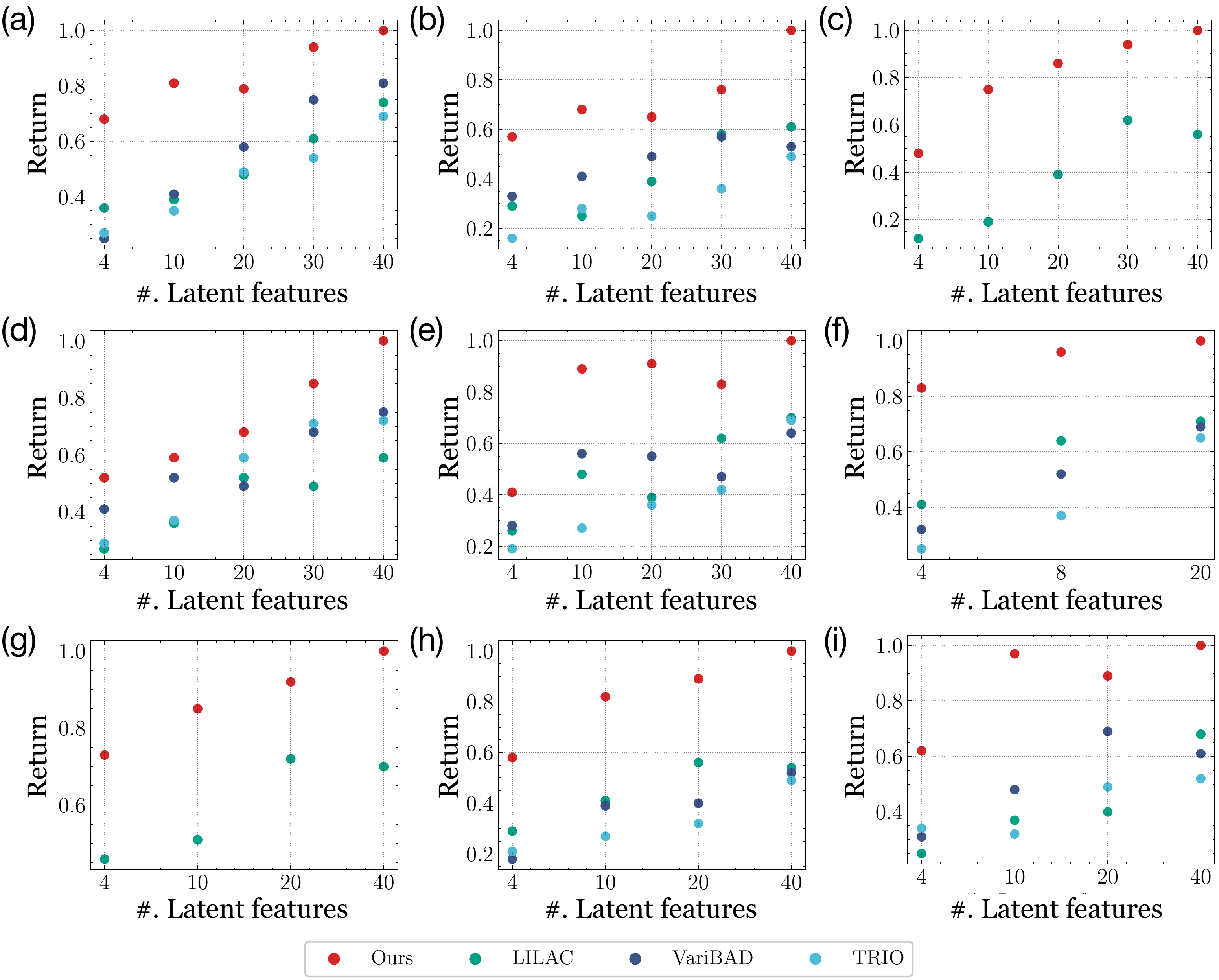}
        \end{center}
        \caption{Average return on different benchmarks with different number of latent features. (a) Half-Cheetah experiments with discrete (across-episode) changes on wind forces; (b) Half-Cheetah experiments with discrete (within-episode) changes on wind forces; (c) Half-Cheetah experiments with continuous changes on wind forces; (d) Half-Cheetah experiments with discrete (across-episode) changes on target speed; (e) Half-Cheetah experiments with discrete (across-episode) changes on wind forces and target speed concurrently; (f) Sawyer-Reaching experiment with discrete (across-episode) changes on target locations; (g) Minitaur experiments with continuous changes on the mass; (h) Minitaur experiments with discrete (across-episode) changes on the target speed; (i) Minitaur experiments with discrete (across-episode) changes on mass and target speed concurrently.
        }
        \label{fig:appendix_latent}
    \end{figure}

    \subsection{Ablation studies on FANS-RL} \label{app:ablation}
    \begin{figure}[H]
        \begin{center}
            \includegraphics[width=0.6\linewidth]{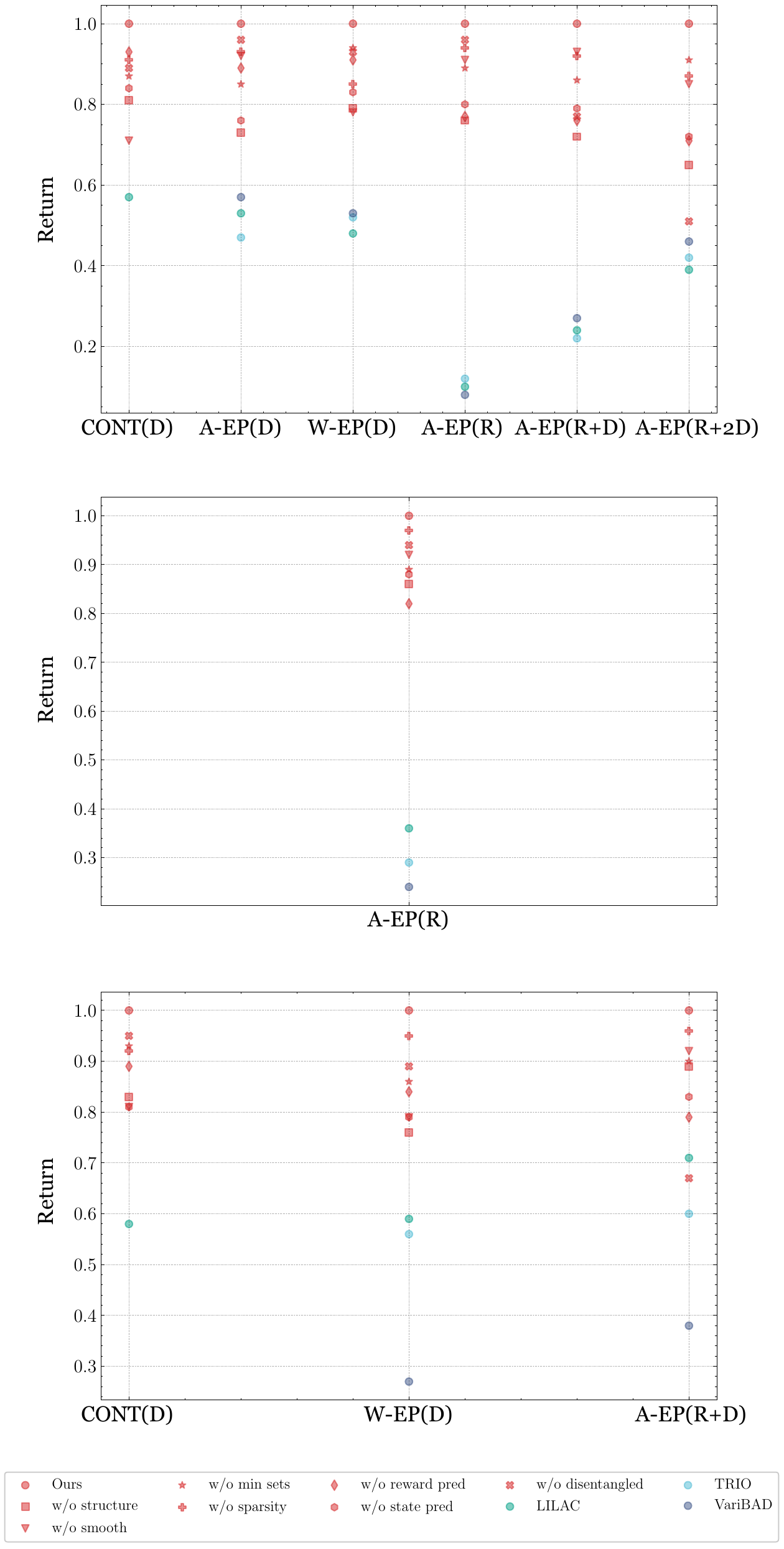}
        \end{center}
        \caption{Ablation studies on different components in FANS-RL on (a) Half-Cheetah experiment; (b) Sawyer experiment; and (c) Minitaur experiments. CONT, A-, W-EP indicate continuous, across-episode, and within-episode changes, respectively. (D) and (R) represent changes on dynamics and reward functions, respectively. Best viewed in color.}
        \label{fig:ablation2}
    \end{figure}
    To verify the effectiveness of each component in our proposed framework, we consider the following ablation studies:

    $\bullet$ Without smoothness loss ($\mathcal{L}_\text{smooth}$);

    $\bullet$ Without structural relationships ($\boldsymbol{C}^{\cdot \scriptveryshortarrow \cdot}$);

    $\bullet$ Without compact representations ($\boldsymbol{s}^{min}, \boldsymbol{\theta}^{min}$);

    $\bullet$ Without sparsity losses ($\mathcal{L}_{\text{sparse}}$);

    $\bullet$ Without reward or state prediction losses ($\mathcal{L}_{\text{pred-rw}}$, $\mathcal{L}_{\text{pred-dyn}}$);

    $\bullet$ Without the disentangled design of CF inference networks for dynamics ($q_{\phi}^s$) and rewards ($q_{\phi}^r$). Specifically, we use one CF inference encoder and the mixed latent space of $\boldsymbol{\theta}^s$ and $\boldsymbol{\theta}^r$ in this setting.

    As shown in Fig.~\ref{fig:ablation2}, all the studied components benefit the performance. Furthermore, FANS-RL can still outperform the strong baselines even without some of the components.

    We also test different smoothness losses, including the moving average (MA) $\mathcal{L}_{\text {smooth }}=\sum_{t=2}^T\left(\left\|\theta_t-\left(\theta_{t-1}+\theta_{t-2}+\ldots+\theta_{t-T}\right) / T\right\|_1\right)$ and exponential moving average (EMA) $\mathcal{L}_{\text {smooth }}=\sum_{t=2}^T\left(\left\|\theta_t-\left(\beta \theta_{t-1}+(1-\beta) \mathbf{v}_{t-2}\right)\right\|_1\right)$, where $\mathbf{v}_t=\beta \theta_t+(1-\beta) \mathbf{v}_{t-1}$ and $\mathbf{v}_0$ is a zero vector. Table~\ref{tbl:smooth_ablation} shows the normalized final results of using different smoothness losses. We can find that different smoothness losses have comparable performances.

    \begin{table}[H]
        \centering
        \begin{tabular}{c|c|c|c}
            \toprule
            & Ours   & MA ($T=2$) & EMA ($\beta = 0.98$) \\ \midrule
            Half-Cheetah: A-EP (D\_1) & $1.00$ & $1.02$     & $0.89$               \\ \hline
            Half-Cheetah: A-EP (D\_2) & $1.00$ & $0.96$     & $0.90$               \\ \hline
            Half-Cheetah: W-EP (D)    & $1.00$ & $0.88$     & $1.05$               \\ \hline
            Half-Cheetah: CONT (D)    & $1.00$ & $1.04$     & $0.95$               \\ \hline
            Half-Cheetah: A-EP (R)    & $1.00$ & $0.93$     & $0.82$               \\ \hline
            Half-Cheetah: A-EP (R+D)  & $1.00$ & $1.09$     & $1.02$               \\ \hline
            Sawyer-Reaching: A-EP (R) & $1.00$ & $0.97$     & $0.91$               \\ \hline
            Minitaur: CONT (D)        & $1.00$ & $1.08$     & $0.96$               \\ \hline
            Minitaur: W-EP (D)        & $1.00$ & $0.86$     & $1.03$               \\ \hline
            Minitaur: A-EP (R+D)      & $1.00$ & $0.97$     & $0.94$               \\ \bottomrule
        \end{tabular}
        \label{tbl:smooth_ablation}
        \caption{Average final return of using different smoothness losses. }
    \end{table}

    \subsection{Visualization on the learned change factors}
    Fig.~\ref{fig:heatmap} gives the visualization on the learned $\boldsymbol{\theta}$ in Half-Cheetah. Fig.~\ref{fig:heatmap}(a-b) show the pairwise Euclidean distance between learned $\boldsymbol{\theta}^r$ and the axes denote the values of change factors on rewards. Similarly, Fig.~\ref{fig:heatmap}(c-d) displays the Euclidean distance between learned $\boldsymbol{\theta}^s$ and the values of change factors on dynamics. The results suggest that there is a positive correlation between the distance of learned $\boldsymbol{\theta}$ versus the true change factors. This can verify that $\boldsymbol{\theta}$ can capture the change factors in the system.

    \begin{figure}[H]
        \begin{center}
            \includegraphics[width=1.0\linewidth]{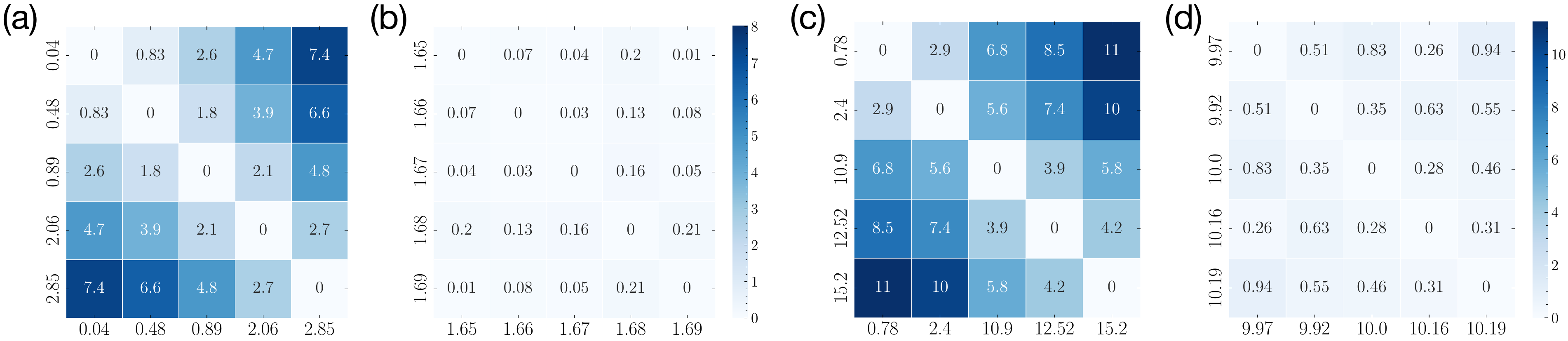}
        \end{center}
        \caption{Visualization on the learned $\boldsymbol{\theta}^r$ and $\boldsymbol{\theta}^s$. Best viewed in color.}
        \label{fig:heatmap}
    \end{figure}

    \section{Details on the Factored Adaptation Framework}

    \subsection{Algorithm pipelines of FN-VAE}
    Alg.~\ref{alg:A1} gives the full pipeline of FN-VAE.
    \begin{algorithm}[h]
        \caption{Learning FN-MDPs using FN-VAE.}
        \label{alg:A1}
        \begin{algorithmic}
            \STATE {\bfseries Input:} Trajectories $\boldsymbol{\tau}$, FN-VAE parameters $\phi = (\phi^s, \phi^r)$, $\alpha = (\alpha_1, \alpha_2)$, $\beta = (\beta_1, \beta_2)$, $\gamma$; Mask matrices $G = (\boldsymbol{C^{\cdot \scriptveryshortarrow \cdot}})$, Boolean updateG, Learning rates $ \lambda_{\phi}, \lambda_{\gamma}, \lambda_{\alpha}, \lambda_{\beta}, \lambda_{G}$, Length of collected rollouts $k$; Number of training epochs E.
            \STATE {\bfseries Output:} $\phi$, $\alpha_1$, $\alpha_2$, $\beta_1$, $\beta_2$, $\gamma$
            \FOR {$i=1,2,\ldots$, E}
            \STATE Randomly sample a batch of trajectories $\boldsymbol{\tau}_{0:k}$ in $\boldsymbol{\tau}$
            \STATE \text{\fontfamily{qcr}\selectfont \# Infer the latent change factors}
            \FOR {j = s, r }
            \STATE Infer $\mu_{\phi^j}(\boldsymbol{\tau}_{0:k})$ and $\sigma^2_{\phi^j}(\boldsymbol{\tau}_{0:k})$ using $q_{\phi^j}$
            \STATE Infer $\mu_{\gamma^j}(\boldsymbol{\theta}^j_{0:k})$  and $\sigma^2_{\gamma^j}(\boldsymbol{\theta}^j_{0:k})$ using $p_{\gamma^j}$
            \STATE Sample $\boldsymbol{\theta}^j_{0: k} \sim \mathcal{N}\left(\mu_{\phi^j}\left(\boldsymbol{\tau}_{0:k}\right), \sigma_{\phi^j}^{2}\left(\boldsymbol{\tau}_{0:k}\right)\right)$
            \ENDFOR
            \STATE Reconstruct and predict $\hat{\boldsymbol{s}}_{0:k}$, $\hat{\boldsymbol{s}}_{1:k}$, $\hat{\boldsymbol{r}}_{0:k}$, $\hat{\boldsymbol{r}}_{1:k}$ using $p_{\alpha_1}$, $p_{\alpha_2}$, $p_{\beta_1}$, and $p_{\beta_2}$
            \STATE \text{\fontfamily{qcr}\selectfont \# Update the FN-VAE model}
            \STATE $\phi \leftarrow \phi - \lambda_{\phi} \nabla_{\phi}{\mathcal{L}}_\text{VAE}$
            \STATE $\gamma \leftarrow \gamma - \lambda_{\gamma} \nabla_{\gamma}\left({\mathcal{L}}_\text{KL} + {\mathcal{L}}_\text{smooth} \right)$
            \STATE $\alpha \leftarrow \alpha - \lambda_{\alpha} \nabla_{\alpha}(\mathcal{L}_\text{rec-dyn} + \mathcal{L}_\text{pred-dyn})$
            \STATE $\beta \leftarrow \beta - \lambda_\beta \nabla_{\beta}(\mathcal{L}_\text{rec-rw} + \mathcal{L}_\text{pred-rw})$
            \IF {updateG}
            \STATE $G \leftarrow G - \lambda_G \nabla_G\left({\mathcal{L}}_\text{rec-dyn} + {\mathcal{L}}_\text{rec-rw} + {\mathcal{L}}_\text{KL} + {\mathcal{L}}_\text{sparse} \right)$
            \ENDIF
            \ENDFOR
        \end{algorithmic}
    \end{algorithm}

    \subsection{The framework dealing with discrete and across-episode changes}
    \label{app: discrete}

    \begin{algorithm}[!h]
        \caption{Factored Adaptation for non-stationary RL (discrete changes.)}
        \label{alg:A3}
        \begin{algorithmic}[1]
            \STATE {\bfseries Init:} Env; VAE parameters: $\phi = (\phi^s, \phi^r)$, $\alpha = (\alpha_1, \alpha_2)$, $\beta = (\beta_1, \beta_2)$, $\gamma$; Mask matrices: $\boldsymbol{C^{\cdot \scriptveryshortarrow \cdot}}$; Policy parameters: $\psi$; replay buffer: $\mathcal{D}$; Number of episodes: $N$; Episode horizon: $H$; Change index $\tilde{\boldsymbol{t}}=\{t_1, \ldots, t_M\}$; $m=0$.
            \STATE {\bfseries Output:} $\phi$, $\alpha_1$, $\alpha_2$, $\beta_1$, $\beta_2$, $\gamma$, $\psi$
            \STATE \text{\fontfamily{qcr}\selectfont \# Model initialization}
            \STATE Collect multiple trajectories $\boldsymbol{\tau} = \{\boldsymbol{\tau}^{1}_{0:k}, \boldsymbol{\tau}^{2}_{0:k}, \ldots\}$ with policy $\pi_{\psi}$ from Env;
            \STATE Learn an initial VAE model on $\boldsymbol{\tau}$ (Alg.~\ref{alg:A1})
            \STATE Identify the compact representations $\boldsymbol{s}^{min}$ and change factors $\boldsymbol{\theta}^{min}$ based on $\boldsymbol{C}^{\cdot \scriptveryshortarrow \cdot}$
            \STATE \text{\fontfamily{qcr}\selectfont \# Model estimation \& policy learning}
            \FOR {$n=0,\ldots, N-1$}
            \FOR{$t=0, \ldots, H-1$}
            \STATE Observe $\boldsymbol{s}_{t}$ from Env;
            \STATE \text{\fontfamily{qcr}\selectfont \# Estimating latent change factors}
            \IF{$n \cdot (H-1)+ t \in \tilde{\boldsymbol{t}}$}
            \STATE $m \leftarrow m+1$
            \FOR {j = s, r }
            \STATE Infer $\mu_{\gamma^j}(\boldsymbol{\theta}^j_{{t}_{m-1}})$ and $\sigma^2_{\gamma^j}(\boldsymbol{\theta}^j_{{t}_{m-1}})$ via $p_{\gamma^j}$
            \STATE Sample ${\boldsymbol{\theta}}^j_{t_m} \sim \mathcal{N}\left(\mu_{\gamma^j}(\boldsymbol{\theta}^j_{{t}_{m-1}}), \sigma_{\gamma^j}^{2}(\boldsymbol{\theta}^j_{{t}_{m-1}})\right)$
            \ENDFOR
            \ENDIF
            \STATE Generate $\boldsymbol{a}_{t} \sim \pi_\psi (\boldsymbol{a}_{t} \mid \boldsymbol{s}^{min}_{t}, \boldsymbol{\theta}^{min}_{t_m})$
            \STATE Receive $r_{n,t}$ from Env
            \STATE Add $(\boldsymbol{s}_{t}, \boldsymbol{a}_{t}, r_t, \boldsymbol{\theta}^s_{t_m}, \boldsymbol{\theta}^r_{t_m})$ to replay buffer $\mathcal{D}$;
            \STATE  Extract a trajectory with length $k$ from $\mathcal{D}$;
            \STATE Learn VAE (Alg.~\ref{alg:A1}) with updateG=False;
            \STATE Sample a batch of data from $\mathcal{D}$
            \STATE Update policy network parameters $\psi$
            \ENDFOR
            \ENDFOR
        \end{algorithmic}
    \end{algorithm}
    Alg.~\ref{alg:A3} gives the extended framework for handling both across- and within-episode changes in non-stationary RL, respectively. The major difference between Alg.~\ref{alg:A3} and Alg.~\ref{alg:A2} is that we only infer $\theta$ using via CF dynamics networks at change points. Furthermore, we also adjust the objective functions of FN-VAE to fit the discrete changes.
    At timestep $t$ in episode $n$, where $t_{m} \leq \big((n-1) \cdot H + t\big) < t_{m+1}$, we have:

    $\bullet$ Prediction and reconstruction losses:
    
    \begin{equation}
        \begin{array}{l}
            \mathcal{L}_{\text{rec-dyn}} =
            \sum \limits_{t=1}^{T-2} \mathbb{E}_{\theta^s_{t_m} \sim q_{\phi}} \log p_{\alpha_1}(\boldsymbol{s}_{t} |\boldsymbol{s}_{t-1}, \boldsymbol{a}_{t-1},\boldsymbol{\theta^s}_{t_m}; \boldsymbol{C}^{\cdot \scriptveryshortarrow \boldsymbol{s}})
            \\
            \mathcal{L}_{\text{pred-dyn}} =  \sum \limits_{t=1}^{T-2} \mathbb{E}_{\theta^s_{t_m} \sim q_{\phi}} \log p_{\alpha_2}(\boldsymbol{s}_{t+1}| \boldsymbol{s}_{t},
            \boldsymbol{a}_{t},\boldsymbol{\theta^s}_{t_m})
            \label{eq:dyn2}
        \end{array}
    \end{equation}
    
    \begin{equation}
        \begin{array}{l}
            \mathcal{L}_{\text{rec-rw}}  = \sum \limits_{t=1}^{T-2} \mathbb{E}_{\theta^r_{t_m} \sim q_{\phi}} \log p_{\beta_1}(r_{t} |\boldsymbol{s}_{t}, \boldsymbol{a}_{t},\boldsymbol{\theta}_{t_m}^r; \boldsymbol{c}^{\boldsymbol{s} \scriptveryshortarrow r}, \boldsymbol{c}^{\boldsymbol{a} \scriptveryshortarrow r})
            \\
            \mathcal{L}_{\text{pred-rw}} = \sum \limits_{t=1}^{T-2} \mathbb{E}_{\theta^r_{t_m} \sim q_{\phi}} \log p_{\beta_2}(r_{t+1}| \boldsymbol{s}_{t+1}, \boldsymbol{a}_{t+1},\boldsymbol{\theta}_{t_m}^r)
            \label{eq:rw2}
        \end{array}
    \end{equation}
    
    $\bullet$ KL loss:
    
    \begin{equation}
        \label{eq:kl2}
        \begin{array}{l}
            \mathcal{L}_{\text{KL}} = \sum \limits_{t=2}^T  \text{KL}  \big( q_{\phi^s}(\boldsymbol{\theta^s}_{t_{m}} |\boldsymbol{\theta^s}_{t_{m-1}}, \boldsymbol{\tau}_{0:t})) \Vert   {p_{\gamma^s}(\boldsymbol{\theta^s}_{t_{m}}|\boldsymbol{\theta^s}_{t_{m-1}}; \boldsymbol{C}^{\boldsymbol{\theta^s} \scriptveryshortarrow \boldsymbol{\theta^s}})} \big) \\
            + \text{KL}  \big( q_{\phi^r}(\boldsymbol{\theta}^r_{t_m} | \boldsymbol{\theta}^r_{t_{m-1}}, \boldsymbol{\tau}_{0:t})) \Vert   {p_{\gamma^r}(\boldsymbol{\theta}^r_{t_m}|\boldsymbol{\theta}^r_{t_{m-1}}; \boldsymbol{C}^{\boldsymbol{\theta}^r \scriptveryshortarrow \boldsymbol{\theta}^r})} \big)
        \end{array}
    \end{equation}
    
    $\bullet$ Sparsity loss:
    
    \begin{equation}
        \begin{aligned}
            \mathcal{L}_{\text{sparse}} =&  w_1\| \boldsymbol{C}^{\boldsymbol{s} \scriptveryshortarrow \boldsymbol{s}} \|_1 + w_2\| \boldsymbol{C}^{\boldsymbol{a} \scriptveryshortarrow \boldsymbol{s}} \|_1 + w_3\| \boldsymbol{C}^{\boldsymbol{\theta^s} \scriptveryshortarrow \boldsymbol{s}} \|_1 \\
            & + w_6\| \boldsymbol{C}^{\boldsymbol{\theta^s} \scriptveryshortarrow \boldsymbol{\theta^s}} \|_1 + w_7\| \boldsymbol{C}^{\boldsymbol{\theta}^r \scriptveryshortarrow \boldsymbol{\theta}^r}\|_1 \\
            & + w_4\| \boldsymbol{c}^{\boldsymbol{s} \scriptveryshortarrow r} \|_1 + w_5\| \boldsymbol{c}^{\boldsymbol{a} \scriptveryshortarrow r} \|_1
        \end{aligned}
        \label{eq:sparse2}
    \end{equation}
    
    $\bullet$ Smoothness loss:
    
    \begin{equation}
        \mathcal{L}_{\text{smooth}} =  \sum
        \limits_{t=2}^T \left(|| \boldsymbol{\theta^s}_{t_m} - \boldsymbol{\theta^s}_{t_{m-1}}||_1 +  ||\boldsymbol{\theta^r}_{t_m} - \boldsymbol{\theta^r}_{t_{m-1}}||_1  \right)
    \end{equation}
    
    The total loss $\mathcal{L}_{\text{vae}} = k_1\mathcal{L}_{\text{rec}} + k_2\mathcal{L}_{\text{pred}} - k_3\mathcal{L}_{\text{KL}} - k_4\mathcal{L}_{\text{sparse}} -
    k_5\mathcal{L}_{\text{smooth}}$, where $k_1$, $k_2$, $k_3$, $k_4$, and $k_5$ are adjustable hyper-parameters to balance the objective functions.

    \subsection{The framework dealing with raw pixels}
    \label{app: pixel}
    We augment the generative process in Eq.~\ref{eq: adarl1}-\ref{Eq: conti_model} with the generative process of observation.
    
    \begin{equation}
    {o}
        _{t} = u_{i} ({c}_{i}^{\boldsymbol{s} \scriptveryshortarrow \boldsymbol{o}} \odot \boldsymbol{s}_{t}, \epsilon_{t}^{o} ),
    \end{equation}
    
    where $u$ is a non-linear function and $i=1, \ldots, d$. $\boldsymbol{c}^{\boldsymbol{s} \scriptveryshortarrow \boldsymbol{o}}:=[c_i^{\boldsymbol{s} \scriptveryshortarrow \boldsymbol{o}}]_{i=1}^d$. $\epsilon_t^o$ is an i.i.d. random noise. To learn the $u_i$, we model the states as the latent variables in FN-VAE.
    \begin{figure}[h]
        \begin{center}
            \includegraphics[width=0.7\linewidth]{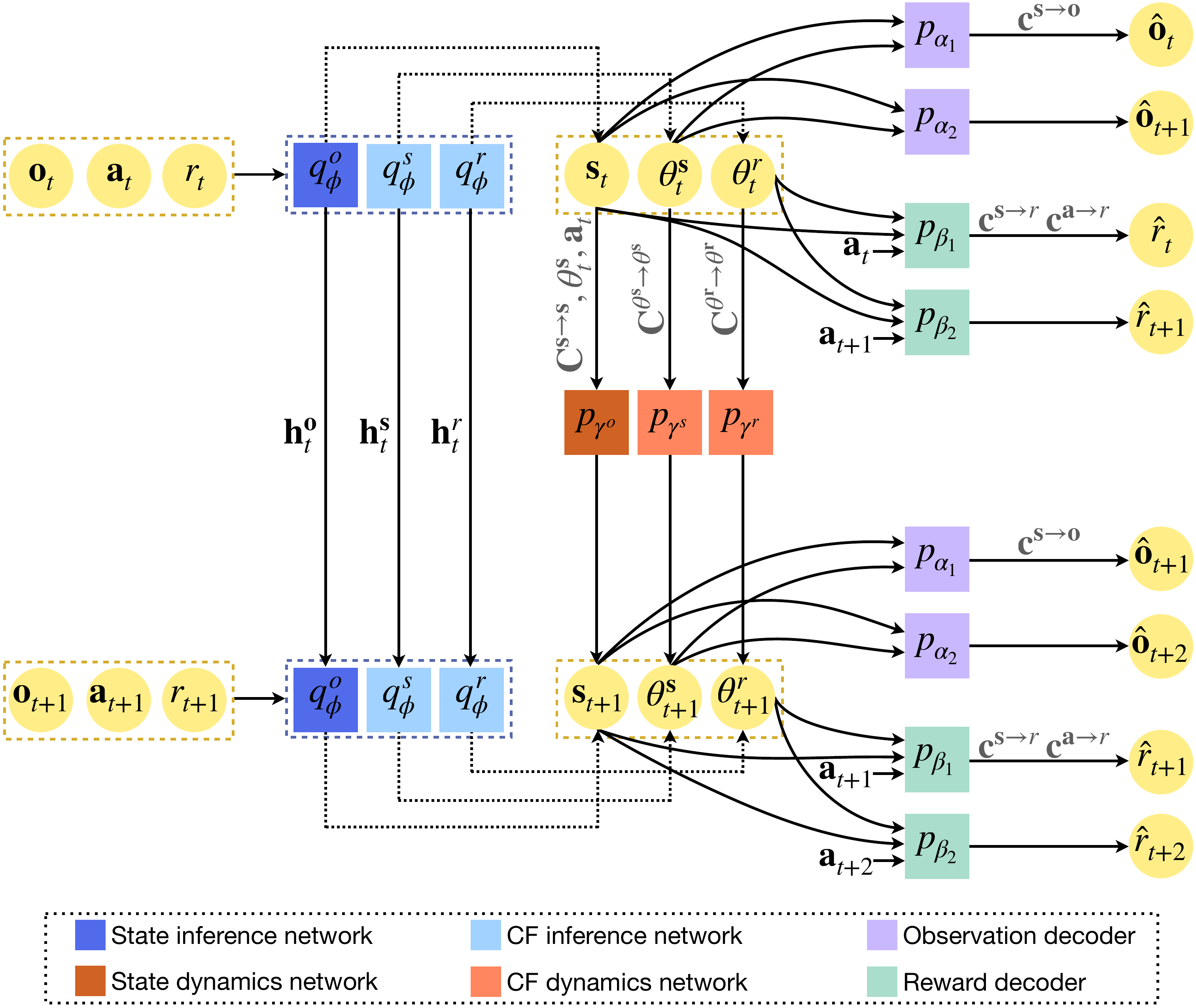}
        \end{center}
        \caption{The architecture of FN-VAE using raw pixel as input.
        }
        \label{fig:pipeline_pomdp}
    \end{figure}
    Fig.~\ref{fig:pipeline_pomdp} gives the modified FN-VAE dealing with raw pixels, where the states are also in the latent space. Different from the original FN-VAE, we incorporate state inference networks and state dynamics networks. Moreover,
    we reconstruct and predict the current and future observations using the observation decoder. Detailed objective functions are given below.\footnote{Here we give the example of handling the discrete changes.}

    $\bullet$ Prediction and reconstruction losses
    
    \begin{equation}
        \begin{array}{l}
            \mathcal{L}_{\text{rec-obs}} =
            \sum \limits_{t=1}^{T-2} \mathbb{E}_{s_t \sim q_{\phi^o}} \log p_{\alpha_1}(\boldsymbol{o}_{t} |\boldsymbol{s}_{t}; \boldsymbol{c}^{\boldsymbol{s} \scriptveryshortarrow \boldsymbol{o}}) 
            \\
            \mathcal{L}_{\text{pred-obs}} =  \sum \limits_{t=1}^{T-2} \mathbb{E}_{{s_t \sim q_{\phi^o}}} \log p_{\alpha_2}(\boldsymbol{o}_{t+1}| \boldsymbol{s}_{t}, \boldsymbol{\theta^s}_{t_m})
            \label{eq:obs}
        \end{array}
    \end{equation}
    
    \begin{equation}
        \begin{array}{l}
            \mathcal{L}_{\text{rec-rw}}  = \sum \limits_{t=1}^{T-2} \mathbb{E}_{(\theta^r_{t_m} \sim q_{\phi}, s_t \sim q_{\phi^o})} \log p_{\beta_1}(r_{t} |\boldsymbol{s}_{t}, \boldsymbol{a}_{t},\boldsymbol{\theta}_{t_m}^r; \boldsymbol{c}^{\boldsymbol{s} \scriptveryshortarrow r}, \boldsymbol{c}^{\boldsymbol{a} \scriptveryshortarrow r})
            \\
            \mathcal{L}_{\text{pred-rw}} = \sum \limits_{t=1}^{T-2} \mathbb{E}_{(\theta^r_{t_m} \sim q_{\phi}, s_t \sim q_{\phi^o})} \log p_{\beta_2}(r_{t+1}| \boldsymbol{s}_t, \boldsymbol{a}_{t+1},\boldsymbol{\theta}_{t_m}^r)
            \label{eq:rw3}
        \end{array}
    \end{equation}
    
    $\bullet$ KL loss
    
    \begin{equation}
        \label{eq:kl3}
        \begin{aligned}
            \mathcal{L}_{\text{KL}} = \sum \limits_{t=2}^T   & \text{KL}  \big( q_{\phi^s}(\boldsymbol{\theta^s}_{t_{m}} |\boldsymbol{\theta^s}_{t_{m-1}}, \boldsymbol{\tau}_{0:t})) \Vert   {p_{\gamma^s}(\boldsymbol{\theta^s}_{t_{m}}|\boldsymbol{\theta^s}_{t_{m-1}}; \boldsymbol{C}^{\boldsymbol{\theta^s} \scriptveryshortarrow \boldsymbol{\theta^s}})} \big) \\
            +  & \text{KL}  \big( q_{\phi^r}(\boldsymbol{\theta}^r_{t_m} | \boldsymbol{\theta}^r_{t_{m-1}}, \boldsymbol{\tau}_{0:t})) \Vert   {p_{\gamma^r}(\boldsymbol{\theta}^r_{t_m}| \boldsymbol{C}^{\boldsymbol{o}_t \scriptveryshortarrow \boldsymbol{\theta}^r})} \big) \\
            +  & \text{KL}  \big( q_{\phi^o}(\boldsymbol{s}_t | \boldsymbol{\tau}_{0:t}, \boldsymbol{\theta}^s_{t_{m}})) \Vert   {p_{\gamma^o}(\boldsymbol{s}_t| \boldsymbol{s}_{t-1}, \boldsymbol{a}_{t-1}, \boldsymbol{\theta}^s_{t_{m}}; \boldsymbol{C}^{\boldsymbol{s} \scriptveryshortarrow \boldsymbol{s}}, \boldsymbol{C}^{\boldsymbol{a} \scriptveryshortarrow \boldsymbol{s}}, \boldsymbol{C}^{\boldsymbol{\theta}^s \scriptveryshortarrow \boldsymbol{s}})} \big)
        \end{aligned}
    \end{equation}
    
    where $\boldsymbol{\tau}_{0:t}=\{\boldsymbol{o}_0, r_0, \boldsymbol{o}_1, r_1, \ldots, \boldsymbol{o}_t, r_t\}$.
    \\
    $\bullet$ Sparsity loss
    
    \begin{equation}
        \begin{aligned}
            \mathcal{L}_{\text{sparse}} =&  w_1\| \boldsymbol{C}^{\boldsymbol{s} \scriptveryshortarrow \boldsymbol{s}} \|_1 + w_2\| \boldsymbol{C}^{\boldsymbol{a} \scriptveryshortarrow \boldsymbol{s}} \|_1 + w_3\| \boldsymbol{C}^{\boldsymbol{\theta^s} \scriptveryshortarrow \boldsymbol{s}} \|_1 \\
            & + w_4\| \boldsymbol{c}^{\boldsymbol{s} \scriptveryshortarrow r} \|_1 + w_5\| \boldsymbol{c}^{\boldsymbol{a} \scriptveryshortarrow r} \|_1 \\
            & + w_6\| \boldsymbol{C}^{\boldsymbol{\theta^s} \scriptveryshortarrow \boldsymbol{\theta^s}} \|_1 + w_7\| \boldsymbol{C}^{\boldsymbol{\theta}^r \scriptveryshortarrow \boldsymbol{\theta}^r}\|_1 + w_8 \| \boldsymbol{C}^{\boldsymbol{s} \scriptveryshortarrow \boldsymbol{o}} \|_1 \\
        \end{aligned}
        \label{eq:sparse3}
    \end{equation}
    
    $\bullet$ Smooth loss
    
    \begin{equation} 
        \mathcal{L}_{\text{smooth}} =  \sum
        \limits_{t=2}^T \left(|| \boldsymbol{\theta^s}_{t_m} - \boldsymbol{\theta^s}_{t_{m-1}}||_1 +  ||\boldsymbol{\theta^r}_{t_m} - \boldsymbol{\theta^r}_{t_{m-1}}||_1  \right)
    \end{equation}

    \begin{figure}[h]
        \centering
        \includegraphics[width=0.7\linewidth]{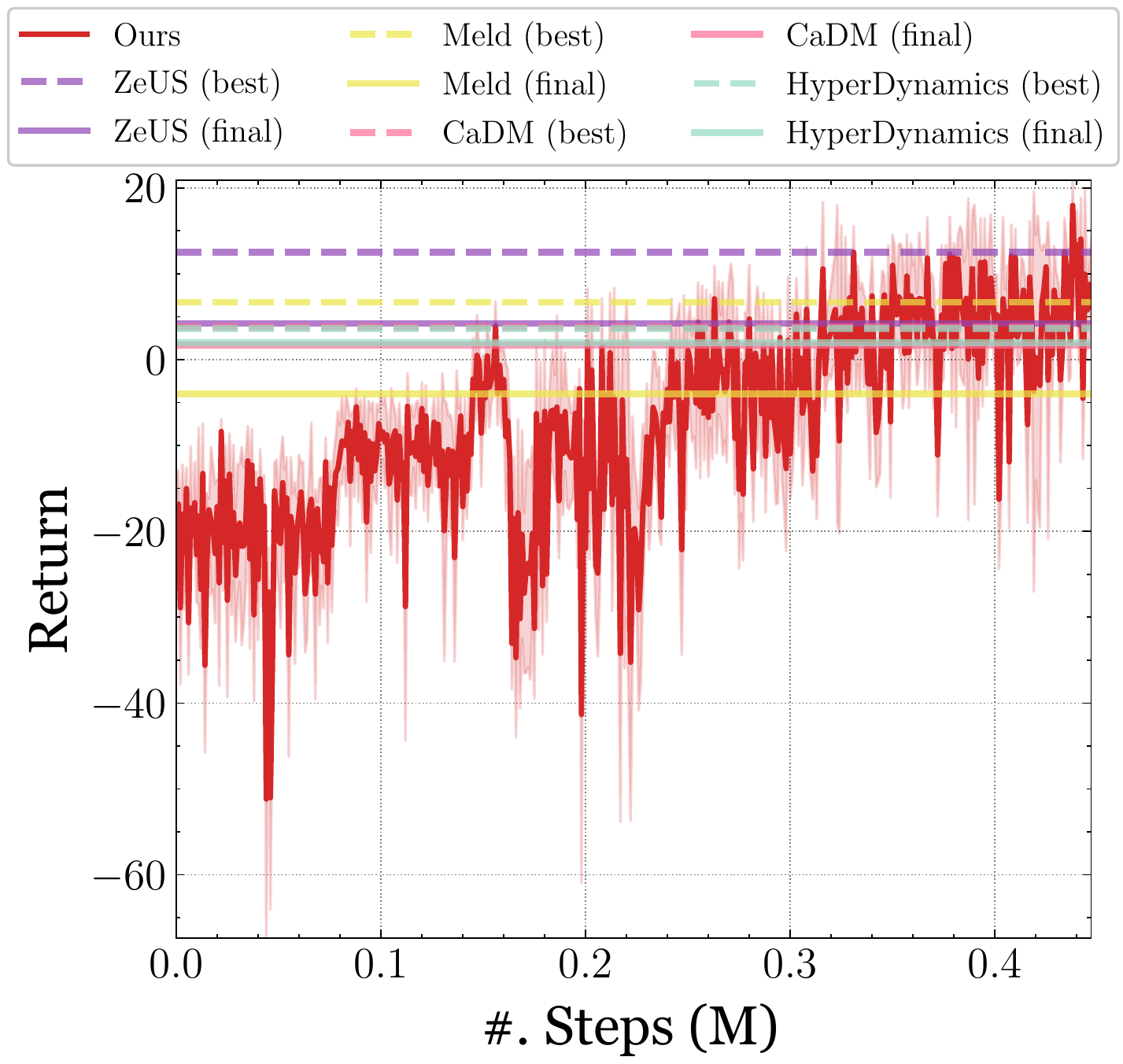}
        \caption{Average return across 10 runs on Sawyer-Peg (raw pixels) with across-episode changes.}
        \label{fig:pomdp}
    \end{figure}

    \subsection{Hyper-parameter selection}\label{app:hyper}

    \subsubsection{Factored model estimation}

    \paragraph{Input with symbolic states.} In Half-Cheetah, Sawyer-Reaching, and Minitaur the symbolic states are observable.
    For the CF dynamic networks, we use $2$-layer fully connected networks. The number of neurons is $512$. For CF inference networks, we use $2$-layer fully connected networks, where the number of neurons is $256$, followed by the LSTM networks with $256$ hidden units. The initial learning rates for all losses are set to be $0.1$ with a decay rate $0.99$. The batch size is $256$ and the length of time steps is equal to the horizon in each task. The number of RNN cells is $256$. The decoder networks are $2$-layer fully connected networks. The number of neurons is $512$.

    \paragraph{Input with raw pixels.} In Saywer-Peg, we directly learn and adapt in non-stationary environments with raw pixels observed. Different from other experiments, we use the architecture described in Fig.~\ref{fig:pipeline_pomdp}. At timestep $t$, we stack $4$ frames as the input $\boldsymbol{o}_t$. A $5$-layer convolutional networks is used to extract the features of the trajectories of observations and rewards.
    The layers have $32$, $64$, $128$, $256$, and $256$ filters. And the corresponding filter sizes are $5$, $3$, $3$, $3$, $4$. The observation decoders are the transpose of the convolutional networks.
    Then the extracted features are used as the input of LSTM networks in state inference networks. The state inference networks and state dynamic networks share the same architectures with the CF inference and dynamics networks, respectively. We use the same CF inference networks, CF dynamics networks, and reward decoders with those in cases with symbolic states as input. The number of latent features is $40$.

    \paragraph{Balancing parameters in losses}
    For all experiments:
    \begin{itemize}
        \item All $w_{\cdot}$ are set to be $0.1$;
        \item Weights of the reconstruction loss: $k_1=0.8$;
        \item Weights of the prediction loss: $k_2=0.8$;
        \item Weights of KL loss: $k_3=0.5$;
        \item Weights of sparsity loss: $k_4=0.1$;
        \item Weights of smooth loss: $k_5=0.02$.
    \end{itemize}

    We use the automatic weighting method in \cite{liebel2018auxiliary} to learn the weights for $k_1, \ldots, K_5$ and grid search for $w_1, \ldots, w_7$.

    \paragraph{Model initialization.}
    Table~\ref{tbl: param_me_hc}, \ref{tbl: param_me_saw}, and \ref{tbl: param_me_minitaur} provide the settings of learning the model initialization.
    \begin{table}[H]
        \centering
        \begin{tabular}{c|c|c|c|c|c}
            \midrule
            & CONT (D) & A-EP (D) & W-EP (D) & A-EP (R) & A-EP (R+D) \\ \hline
            \# trajectories          & $500$    & $20$     & $20$     & $20$     & $100$      \\ \hline
            \# steps in each episode & $50$     & $50$     & $50$     & $50$     & $50$       \\ \hline
            \# episodes              & $10$     & $100$    & $100$    & $100$    & $100$      \\ \bottomrule
        \end{tabular}
        \caption{The selected hyper-parameters for model estimation in Half-Cheetah experiment.}
        \label{tbl: param_me_hc}
    \end{table}
    \begin{table}[H]
        \centering
        \begin{tabular}{c|c|c}
            \midrule
            & Sawyer-Reaching & Sawyer-Peg \\ \hline
            \# trajectories          & $500$           & $20$       \\ \hline
            \# steps in each episode & $150$           & $40$       \\ \hline
            \# episodes              & $10$            & $100$      \\ \bottomrule
        \end{tabular}
        \caption{The selected hyper-parameters for model estimation in Saywer experiments.}
        \label{tbl: param_me_saw}
    \end{table}
    \begin{table}[H]
        \centering
        \begin{tabular}{c|c|c|c}
            \midrule
            & CONT (D) & W-EP (D) & A-EP (R+D) \\ \hline
            \# trajectories          & $500$    & $50$     & $80$       \\ \hline
            \# steps in each episode & $100$    & $100$    & $100$      \\ \hline
            \# episodes              & $10$     & $50$     & $100$      \\ \bottomrule
        \end{tabular}
        \caption{The selected hyper-parameters for model estimation in Minitaur experiments.}
        \label{tbl: param_me_minitaur}
    \end{table}

    \subsubsection{Policy learning}
    In the Half-Cheetah, Sawyer-Reaching, and Minitaur experiments, we follow the learning rates selection for policy networks in~\citep{xie2020deep}.
    In Sawyer-Peg, for both actor and critic networks, we use $2$-layer fully-connected networks. The number of neurons is $256$. For all experiments, we use standard Gaussian to initialize the parameters of policy networks. The learning rate is $3e-4$. The relay buffer capacity is $50,000$. The number of batch size is $256$.

    \paragraph{Details on TRIO and VariBAD.}

    For TRIO and VariBAD, we meta-train the models (batch size: 5000, \# epochs: 2 for all experiments) and show the learning curves of meta-testing. The tasks parameters for meta-training are uniformly sampled from a Gaussian distribution. For all approaches, we use the same set of hyper-parameters for policy optimization modules (i.e., SAC). For the latent parameters in TRIO, we follow the original paper where the latent space from the inference network is projected to a higher dimension. The number of latent parameters for TRIO is the same as those in other approaches (Half-Cheetah and Minitaur: $40$, Sawyer-Reaching: $20$). We compared with TS-TRIO, with the kernels set as in the original implementation.

    \section{Experimental Platforms and Licenses}

    \subsection{Platforms}
    \label{app:platforms}
    All methods are implemented on 8 Intel Xeon Gold 5220R and 4 NVidia V100 GPUs.

    \subsection{Licenses}\label{app:license}
    In our code, we have used the following libraries which are covered by the corresponding licenses:
    \begin{itemize}
        \item Tensorflow (Apache License 2.0),
        \item Pytorch (BSD 3-Clause "New" or "Revised" License),
        \item OpenAI Gym (MIT License),
        \item OpenCV (Apache 2 License),
        \item Numpy (BSD 3-Clause "New" or "Revised" License)
        \item Keras (Apache License).
    \end{itemize}
\end{document}